\newif\ifarxiv 
\arxivtrue
\ifarxiv
\documentclass{article}
\usepackage[preprint]{neurips_2023}

\usepackage{authblk}

\usepackage[utf8]{inputenc} 
\usepackage[T1]{fontenc}    
\usepackage{hyperref}       
\usepackage{url}            
\usepackage{booktabs}       
\usepackage{amsfonts}       
\usepackage{nicefrac}       
\usepackage{microtype}      
\usepackage{xcolor}         
\usepackage{subfig}
\title{From Self-Attention to Markov Models:\\Unveiling the Dynamics of Generative Transformers}
\author{ M. Emrullah Ildiz$^1$ \quad Yixiao Huang$^1$ \quad  Yingcong Li$^1$\\
{Ankit Singh Rawat$^2$ \quad Samet Oymak$^1$}
}
\affil{$^1$ University of Michigan, Ann Arbor \\
{\small \texttt{\{eildiz,yingcong,oymak\}@umich.edu, yixiao.huang@my.cityu.edu.hk} 
\vspace{0mm}}
}
\affil{$^2$ Google Research NYC \\
{\small \texttt{ankitsrawat@google.com} 
\vspace{-3mm}}
} 
\else
\documentclass{article}
\usepackage{icml2024}
\usepackage{microtype}
\usepackage{graphicx}
\usepackage{subfigure}
\usepackage{changepage}
\usepackage{booktabs} 

\usepackage{hyperref}

\usepackage{amsmath}
\usepackage{amssymb}
\usepackage{mathtools}
\usepackage{amsthm}
\usepackage{wrapfig}
\usepackage[capitalize,noabbrev]{cleveref}
\icmltitlerunning{Understanding Self-Attention through PCMC}
\fi

\newtheorem{theorem}{Theorem}[section]

\newtheorem{lemma}[theorem]{Lemma}
\newtheorem{corollary}[theorem]{Corollary}
\newtheorem{definition}[theorem]{Definition}
\newtheorem{assumption}[theorem]{Assumption}

\newcommand{\rulesep}{\unskip\ \vrule\ }
\usepackage[textsize=tiny]{todonotes}
\usepackage{changepage}

\usepackage[utf8]{inputenc} 
\usepackage{natbib}
\usepackage{xspace}
\usepackage[T1]{fontenc}    
\usepackage{hyperref}       
\usepackage{url}            
\usepackage{booktabs}       
\usepackage{amsfonts,amsmath,amssymb}       
\usepackage{nicefrac}       
\usepackage{microtype}      
\usepackage{xcolor}         
\usepackage{enumitem}
\usepackage{mathtools}

\usepackage{graphicx}
\usepackage[normalem]{ulem}
\usepackage{ragged2e}

\usepackage{wrapfig,bm,comment,color}
\usepackage{breakurl,epsfig,epsf,fmtcount,semtrans,multirow,multicol,boldline}
\usepackage{tcolorbox}
\tcbuselibrary{skins}
\usepackage{tikz}
\usepackage{titletoc}
\usepackage{pifont}

\usepackage{appendix}
\usepackage{xcolor}

\definecolor{darkred}{RGB}{150,0,0}
\definecolor{darkgreen}{RGB}{0,150,0}
\definecolor{darkblue}{RGB}{0,0,200}
\hypersetup{colorlinks=true, linkcolor=darkred, citecolor=darkgreen, urlcolor=darkblue}

\def \endprf{\hfill {\vrule height6pt width6pt depth0pt}\medskip}

\newtheorem{observation}{Observation}
\newenvironment{proof}{\noindent {\bf Proof.} }{\endprf\par}

\newcommand{\bpij}{\pi_{ij}}

\newcommand{\red}{\textcolor{darkred}}

\newcommand{\SO}[1]{\textcolor{blue}{[SO: #1]}}

\newcommand{\cln}[1]{\red{}}

\newcommand{\eps}{\varepsilon}

\newcommand{\hb}{\vct{h}}

\newcommand{\st}{\star}

\newcommand{\Ace}{\Ac_{\eps}}
\newcommand{\Sall}{\Sc_{\Eb}}
\newcommand{\Sprm}{\Sc_{\data}}

\newcommand{\bPi}{\boldsymbol{\Pi}}

\newcommand{\beq}{\begin{equation}}
\newcommand{\ba}{\begin{align}}
\newcommand{\ea}{\end{align}}

\newcommand{\eeq}{\end{equation}}

\newcommand{\nn}{\nonumber}

\newcommand{\A}{{\mtx{A}}}

\newcommand{\Ub}{{\mtx{U}}}

\newcommand{\B}{{{\mtx{B}}}}

\newcommand{\Ib}{{{\mtx{I}}}}

\newcommand{\Sb}{{{\mtx{S}}}}

\newcommand{\diag}[1]{\text{diag}(#1)}

\newcommand{\Lc}{{\cal{L}}}
\newcommand{\Oc}{{\cal{O}}}
\newcommand{\Ocb}{\bar{\Oc}}

\newcommand{\Lch}{{\widehat{\cal{L}}}}

\newcommand{\Nc}{{\cal{N}}}
\newcommand{\Sp}{{\Sc_{\bpigrd}}}
\newcommand{\bpist}{\bpi_*}
\newcommand{\pigrd}{\pi^{\texttt{GT}}}
\newcommand{\pist}{\pi_*}
\newcommand{\bpigrd}{\bpi^{\texttt{GT}}}

\newcommand{\Dc}{{\cal{D}}}

\newcommand{\Pb}{{\mtx{P}}}

\newcommand{\Pbh}{\hat{{\mtx{P}}}}
\newcommand{\Pbs}{{\mtx{P}^*}}

\newcommand{\Qb}{{\mtx{Q}}}

\newcommand{\Cb}{{\mtx{C}}}

\newcommand{\Eb}{{\mtx{E}}}

\newcommand{\Gc}{{\cal{G}}}

\newcommand{\Iden}{{\mtx{I}}}
\newcommand{\M}{{\mtx{M}}}

\newcommand{\order}[1]{{\cal{O}}(#1)}

\newcommand{\z}{{\vct{z}}}

\newcommand{\sft}[1]{\mathbb{S}(#1)}
\newcommand{\sfp}[1]{\mathbb{S}'(#1)}

\newcommand{\tnt}[1]{\|{#1}\|_2}

\newcommand{\ti}[1]{\|{#1}\|_{\infty}}

\newcommand{\tf}[1]{\|{#1}\|_{F}}

\newcommand{\Ac}{\mathcal{A}}

\newcommand{\bt}{{\boldsymbol{\beta}}}

\newcommand{\bal}{{\boldsymbol{\alpha}}}
\newcommand{\bgam}{{\boldsymbol{\gamma}}}

\newcommand{\Sc}{\mathcal{S}}

\newcommand{\vb}{\vct{v}}
\newcommand{\nb}{\vct{n}}

\newcommand{\xb}{\vct{\bar{x}}}
\newcommand{\Xb}{\vct{\bar{X}}}

\newcommand{\xo}{\vct{x}^{occ}}

\newcommand{\cb}{\mtx{c}}

\newcommand{\w}{\vct{w}}

\newcommand{\s}{\vct{s}}
\newcommand{\ab}{\vct{a}}
\newcommand{\bb}{\vct{b}}
\newcommand{\ub}{{\vct{u}}}

\newcommand{\Tc}{\mathcal{T}}

\newcommand{\Fc}{\m}

\newcommand{\Xc}{\mathcal{X}}
\newcommand{\Yc}{\mathcal{Y}}

\newcommand{\m}{\vct{m}}
\newcommand{\mb}{\bar{m}}
\newcommand{\bmb}{\bar{\vct{m}}}

\newcommand{\Ws}{\W^\star}

\newcommand{\data}{\Tc}

\newcommand{\x}{\vct{x}}
\newcommand{\xl}{\x_L}
\newcommand{\xxl}{x_L}
\newcommand{\xli}{\x_{i,L_i}}
\newcommand{\xlj}[1]{\x_{L#1}}
\newcommand{\maskb}{\m}
\newcommand{\mask}{{m}}
\newcommand{\xlp}{\bar{ x}}
\newcommand{\xlpi}{\bar{ x}_i}

\newcommand{\y}{\vct{y}}
\newcommand{\Dcxy}{\Dc_{\Xc\Yc}}
\newcommand{\Dcx}{\Dc_{\Xc}}

\newcommand{\W}{\mtx{W}}
\newcommand{\Wout}{\mtx{W}_{out}}

\newcommand{\Winn}{\mtx{W}_{inn}}

\newcommand{\Wh}{\hat{\mtx{W}}}
\newcommand{\Wgrd}{\mtx{W}^{\texttt{GT}}}
\newcommand{\Wst}{\mtx{W}_\st}
\newcommand{\Pst}{\mtx{P}_\st}
\newcommand{\Pgrd}{\mtx{P}^{\texttt{GT}}}
\newcommand{\Pgrdnb}{P^{GT}}

\newcommand{\bgl}{{~\big |~}}

\newcommand{\bpi}{\boldsymbol{\pi}}

\definecolor{emmanuel}{RGB}{255,127,0}

\newcommand{\PCMC}{{\textrm{CCMC}}\xspace}

\newcommand{\pb}{{\vct{p}}}
\newcommand{\Pbw}{{\vct{P}}^{\W}}
\newcommand{\qb}{{\vct{q}}}
\newcommand{\R}{\mathbb{R}}
\newcommand{\Pro}{\mathbb{P}}

\newcommand{\<}{\langle}
\renewcommand{\>}{\rangle}

\renewcommand{\P}{\operatorname{\mathbb{P}}}
\newcommand{\E}{\operatorname{\mathbb{E}}}

\newcommand{\e}{\mathrm{e}}
\newcommand{\eb}{\vct{e}}
\newcommand{\emax}{\vct{e}_{max}}

\newcommand{\vct}[1]{\bm{#1}}
\newcommand{\mtx}[1]{\bm{#1}}

\newcommand{\Pc}{{\cal{P}}}
\newcommand{\X}{{\mtx{X}}}

\newcommand{\Vb}{{\mtx{V}}}

\newif\ifdraft
\drafttrue

\ifdraft
    \newcommand{\asrtodo}[1]{\todo[color=cyan,size=\tiny]{AR: #1}}
    \newcommand{\asr}[1]{\textsf{\textcolor{red}{[\textbf{Ankit}: #1]}}}

    \newcommand{\ylm}[1]{\marginpar{\color{orange}\tiny\ttfamily YL: #1}}
    \newcommand{\ankit}[1]{\textsf{\textcolor{red}{[\textbf{ASR}: #1]}}}
    \newcommand{\shaw}[1]{\textcolor{violet}{#1}}
    \newcommand{\redp}[1]{\textcolor{darkred}{[#1]}}

\else 
    \newcommand{\asrtodo}[1]{}
    \newcommand{\asr}[1]{}

    \newcommand{\ylm}[1]{}
    \newcommand{\ankit}[1]{}
    \newcommand{\shaw}[1]{}
    \newcommand{\redp}[1]{{[#1]}}

\fi
\date{}
\begin{document}

\ifarxiv
\maketitle

\fi
\begin{abstract}
Modern language models rely on the transformer architecture and attention mechanism to perform language understanding and text generation. In this work, we study learning a 1-layer self-attention model from a set of prompts and associated output data sampled from the model. We first establish a precise mapping between the self-attention mechanism and Markov models: Inputting a prompt to the model samples the output token according to a \emph{context-conditioned Markov chain} (CCMC) which weights the transition matrix of a base Markov chain. Additionally, incorporating positional encoding results in position-dependent scaling of the transition probabilities. Building on this formalism, we develop identifiability/coverage conditions for the prompt distribution that guarantee consistent estimation and establish sample complexity guarantees under IID samples. Finally, we study the problem of learning from a single output trajectory generated from an initial prompt. We characterize an intriguing \emph{winner-takes-all} phenomenon where the generative process implemented by self-attention collapses into sampling a limited subset of tokens due to its non-mixing nature. This provides a mathematical explanation to the tendency of modern LLMs to generate repetitive text. In summary, the equivalence to CCMC provides a simple but powerful framework to study self-attention and its properties.

\end{abstract}

\section{Introduction}

The attention mechanism~\cite{vaswani2017} is a key component of the canonical transformer architecture which underlies the recent advances in language modeling~\citep{radford2018_gpt1, radford2019_gpt2, brown2020_gpt3, chowdhery2022_palm, touvron2023_llama}. The self-attention layer allows all tokens within an input sequence to interact with each other. Through these interactions, the transformer assesses the similarities of each token to a given query and composes their value embedding in a non-local fashion. 

In this work, we study the mathematical properties of the one-layer self-attention model where the model is trained to predict the next token of an input sequence. The token generation process via the self-attention mechanism is non-trivial because the generation depends on \textit{entire} input sequence. This aspect is crucial for the capabilities of modern LLMs where the response is conditioned on the user prompt. It is also unlike well-understood topics such as Markov Chains where the model generates the next state based on the current one. Theoretical analysis is further complicated by the fact that the optimization landscape is typically nonconvex. This motivates us to ask:
\begin{adjustwidth}{10pt}{}
\begin{quote}
\textbf{Q:} Can self-attention be formally related to fundamental models such as Markov chains? Can this allow us to study its optimization, approximation, and generalization properties?
\end{quote}
\end{adjustwidth}

Our main contribution is addressing this question by formally mapping the generative process of one self-attention layer to, what we call, \emph{Context-Conditioned Markov Chains (\PCMC)} under suitable conditions. In essence, \PCMC modifies the transition probabilities of a \emph{base Markov chain} according to the sequence of tokens/states observed so far. Thus, learning a 
self-attention layer from the (prompt, output) pairs generated by it can be interpreted as learning a Markov chain from its \textit{context-conditioned} transitions. 

Concretely, we make the following contributions:
\begin{itemize}[leftmargin=4mm, itemsep=1mm, partopsep=0pt,parsep=0pt]
\item \textbf{\PCMC$\Leftrightarrow$ Self-attention (Sec \ref{sect setup}). }We introduce \PCMC and show that it can precisely represent the transition dynamics of self-attention under suitable conditions. Importantly, the optimization of self-attention weights becomes convex, hence tractable via gradient descent, under maximum likelihood estimation ($-\log$ loss).
\item \textbf{Consistency of learning (Sec \ref{sect consistency}).} We study the learnability of a self-attention layer where we observe its outputs for a set of input prompts. 
In practice, this is motivated by the question: \emph{Can we distill the generative capabilities of a language model by collecting its outputs on a set of instructions/prompts?} Through the \PCMC connection, we identify necessary and sufficient \emph{coverage conditions} 
on the prompt distribution that ensures consistent estimation of the underlying model.
\item \textbf{Sample complexity (Sec \ref{sect finite sample}).} Integrating consistency guarantees with finite sample analysis, we develop generalization guarantees for learning a ground-truth self-attention model from its IID (prompt, output) pairs. We establish a fast statistical rate of $\order{K^2/n}$ where $K$ is the size of the token vocabulary and $n$ is the sample size. 
\item \textbf{Learning from single prompt trajectory (Sec \ref{sect single traj}).} Going beyond IID samples, we provide theory and experiments on the learnability of self-attention from a \emph{single trajectory} of its autoregressive generation. Our findings reveal a \emph{distribution collapse} phenomenon where the transition dynamics evolve to generate only one or very few tokens while suppressing the other tokens. This also provides an explanation to why modern LLMs tend to generate repetitive sentences after a while \citep{see2017get, holtzman2019curious, xu2022learning}. Finally, we study the characteristics of self-attention trajectory, identify novel phase transitions, and shed light on when consistent estimation succeeds or fails. 
\item \textbf{The role of positional encoding (Sec \ref{sect position embedding}).} We augment our theory to incorporate positional encoding (PE). We show that PE enriches \PCMC to make transition dynamics adjustable by learnable positional priors.
\end{itemize}

Overall, we believe that CCMC provides a powerful and rigorous framework to study self-attention and its characteristics. Note that, the token generation process implemented by self-attention is inherently non-Markovian because the next generated token depends on the whole past input prompt/trajectory. Thus, CCMC is also a non-Markovian process, however, it admits a simple representation in terms of a \emph{base Markov chain} and the prompt/trajectory characteristics. CCMC is illustrated in Figure \ref{fig:main_figure2} and formally introduced in the next section together with its connection to the self-attention mechanism.
\section{Setup: Markov Chain and Self-Attention}\label{sect setup}
\textbf{Notation. } Let $[n]$ denote the set $\{1,\cdots,  n\}$ for an integer $n \geq 1$.
We use lower-case and upper-case bold letters (e.g., $\ab, \A$) to represent vectors and matrices, respectively. $a_i$ denotes the $i$-th entry of the vector $\ab$. Let $\bPi_\Sc(\cdot)$ denote the projection operator on a set $\Sc$, $\mathbf{1}(E)$ denote the indicator function of an event $E$, and $\sft{\cdot}:\R^L\rightarrow\R^L$ denote the softmax operation. We use $\lesssim, \gtrsim$ for inequalities that hold up to constant/logarithmic factors.

\subsection{Context-Conditioned Markov Chain (\PCMC)}
Let $\Pb=[\bpi_1 \dots \bpi_K] \in \R^{K \times K}$ be the transition matrix associated with a base Markov chain where the $i$-th column $\bpi_i = (\pi_{i1},\ldots, \pi_{iK})\in\R^K$ captures the transition probabilities from state $i\in[K]$ with entries adding up to $1$. Thus, given random state sequence $(x_t)_{t\geq 1}$ drawn according to $\Pb$, we have that $\P(x_{t+1} = j | x_{t} = i ) = \bpij$.

Now consider the modified transitions for $\Pb$ where transition probabilities are weighted according to a vector $\m\in\R^K$ with non-negative entries. Concretely, we consider the following transition model
\begin{align}\label{eqn masked markov chain}
    \P_{\maskb}(x_{t+1} = j | x_{t} = i ) = \pi^{\m}_{ij}:=\frac{m_j \cdot \bpij}{ \maskb^\top \bpi_i}.
\end{align} 
Note that this transition model is still a standard Markov chain with updated transition probabilities $\pi^{\m}_{ij}=\frac{m_j \cdot \bpij}{ \maskb^\top \bpi_i}$. In contrast, we now introduce the setting where weighting $\m$ changes as a function of the input sequence.

\textbf{Context-conditioned Markov Chain (\PCMC).} \PCMC is a \textit{non-Markovian} transition model derived from a base transition matrix $\Pb$. To proceed, given a state trajectory $X=(x_t)_{t=1}^L\in [K]^L$, let us define $\Fc(X)$ to be the empirical frequencies of individual states where
\begin{align}
\label{eq:freq-func}
\Fc(X)_k=\frac{|\{t \in [L]~:~x_t = k\}|}{L},\quad \forall k \in [K].
\end{align}
\PCMC is obtained by weighting the standard Markov chain transitions according to $\Fc(X)$ determined by $X$.
\begin{definition}\label{def pcmc}Let $X=(x_t)_{t=1}^L$ and $\maskb=\Fc(X)$. Given a transition matrix $\Pb$, the associated \PCMC transition from state $x_L$ to $x_{L+1}$ is governed by $\bpi^X \triangleq \bpi^{\Fc(X)}_{x_L} \in\R^K$ defined as
\begin{align}\label{PCMC transition}
    \P_{\Pb}(x_{L+1} = j | X ) = \pi^X_j:=\frac{\mask_j \cdot \pi_{\xxl, j}}{ \maskb^\top \bpi_{\xxl}}.
\end{align} 
\end{definition}
Here, note that the last element $x_L$ of $X$ still serves as the state of the Markov chain; however, transitions are weighted by state frequencies, which can be observed in Figure \ref{fig:main_figure2}. These frequencies will be evolving as the model keeps generating new states. In the context of language modeling, $X=(x_t)_{t=1}^L$ will correspond to the \emph{prompt} that we input to the model and $(x_{t})_{t> L}$ will be the model's response: Sections \ref{sect consistency} and \ref{sect finite sample} will explore the learnability of underlying dynamics $\Pb$ from multiple diverse prompts and the corresponding model generations. On the other hand, Section \ref{sect single traj} will study learnability from an infinite trajectory generated from a single prompt.



As we shall see, Definition \ref{def pcmc} captures the dynamics of a 1-layer self-attention model when there are no positional encodings. In Section \ref{sect position embedding}, we introduce a more general setting where the transition dynamics of \PCMC incorporates the positional information of the state trajectory. This enriched model will similarly capture self-attention with absolute positional encoding.



\subsection{Attention-based Token Generation}
\begin{figure*}[tb] 
   \centering    \includegraphics[width=0.9\textwidth]{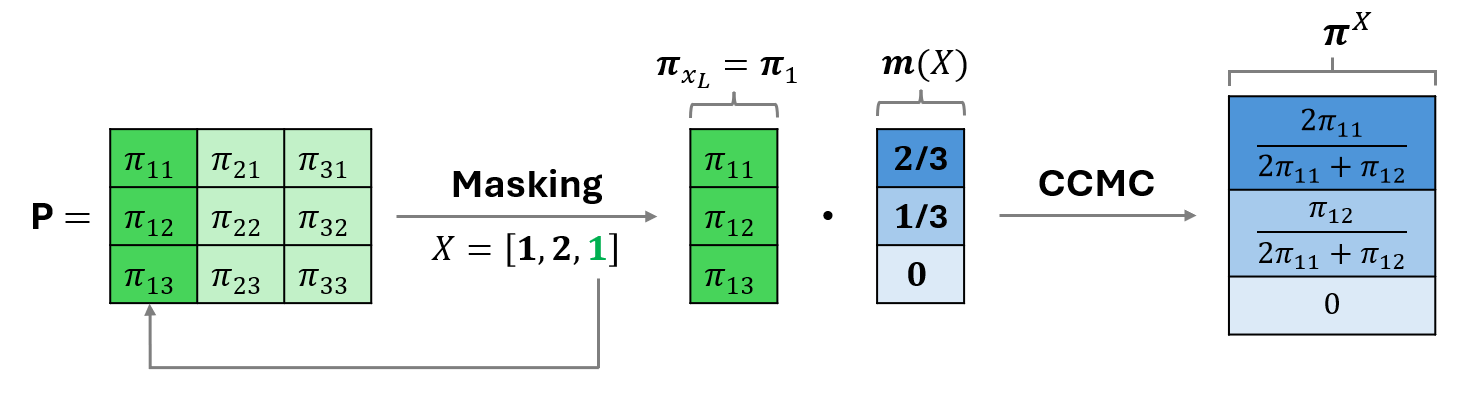}
   \vspace{-10pt}
   \caption{Demonstration of Definition \ref{def pcmc}. We provide an example where the vocabulary size $K = 3$ and the input prompt $X = [1, 2, 1]$, which results in a frequency vector $\Fc(X)$. $\Pb$ represents the transition matrix of the base Markov chain.
   }
       \vspace{-10pt}
   \label{fig:main_figure2}
\end{figure*}
We consider a single attention head which admits a sequence of tokens and outputs the next token. Suppose we have a vocabulary of $K$ tokens denoted by $[K]$, which precisely corresponds to the set of states in the (base) Markov chain. To feed tokens to the attention layer, we embed them to get an embedding matrix $\Eb=\begin{bmatrix}\eb_1~\dots~\eb_K\end{bmatrix}^\top\in\R^{K\times d}$, where $\eb_i\in\R^d$ is the embedding of the $i$-th token. To feed a prompt $X=(x_i)_{i=1}^L$ to the attention layer, we first obtain its embedding as follows:
\[ 
\X=[\x_1\dots\x_L]^\top\in\R^{L\times d}\quad\text{where}\quad \x_i:=\eb_{x_i},\forall i\in[L].
\]
Throughout, we consistently denote the embedding of a discrete sequence $X$ and token $x_i$ by $\X$ and $\x_i$, respectively.

\noindent\textbf{Self-attention model.} A single-layer self-attention head predicts the next token based on the input prompt $X$, with the last token $x_L$ forming the query token in the attention layer and playing a distinct role in sampling the next token $x_{L+1}$. Let us denote the combined key-query weights by a trainable matrix $\W \in\R^{d\times d}$, and assume the value weights to be the identity matrix, i.e., $\Vb = \Ib_{d}$. With these, the self-attention layer $f_{\W}$ outputs
\begin{align} 
\label{eq:sattn-def}
f_{\W}(X)&=\X^\top\s_X\quad\text{where}\quad\s_X=\sft{\X\W\xl}.\tag{SA}
\end{align}
Here $\sft{\cdot}$ is the softmax function and $\s_X\in\R^L$ is the softmax probability output associated with $X$. A useful observation is that $f_{\W}(X)$ produces probability-weighted combination of the input tokens. To proceed, let us define a transition matrix $\Pbw\in\R^{K\times K}$ associated with the attention model weights $\W$ as follows:
\begin{align}
\Pbw=[\bpi_1~\dots~\bpi_K] \quad\text{where}\quad \bpi_i=\sft{\Eb\W\eb_i}.\label{pbw def}
\end{align}
Next, we observe the following identity on self-attention. 

\begin{lemma}\label{lemma identity} Let $(X, y)$ be an arbitrary pair of (prompt, next token). Define $\bpi^X \in \R^K$ based on $\Pb^{\W}$ using Definition \ref{def pcmc}. We have that
\[ 
f_{\W}(X)=\X^\top\s_X=\Eb^\top \bpi^X.
\]
\end{lemma}
Lemma~\ref{lemma identity} highlights a fundamental connection between self-attention and \PCMC, which we leverage by defining the following sampling. The proof is provided in Appendix \ref{lemma app identity}.

\noindent\textbf{Sampling-from-softmax.} 
The idea is sampling the next token proportional to its contribution to the output of the self-attention layer. This is equivalent to sampling the next token according to its total probability within the softmax-attention map given by $\bpi^X$. Thus, \emph{sampling-from-softmax} with weights $\W\in\R^{d\times d}$ is mathematically equivalent to a \PCMC with transition dynamics \eqref{pbw def}.

In what follows, we will introduce and investigate an attention-based next token generation model that implements the \emph{sampling-from-softmax} procedure. Let $\Cb\in \R^{K\times d}$ be the linear prediction head. Following attention output $f_{\W}(X)$, we sample the next token from $\Cb f_{\W}(X)\in\R^K$. We will utilize the following assumption.

\begin{assumption}\label{assume iden} 
The vocabulary embeddings $(\eb_k)_{k=1}^K$ are linearly independent and the classifier obeys $\Cb\Eb^\top=\Iden_{K}$.
\end{assumption}

This assumption is a slightly stronger version of the weight-tying, where the output and input embedding are the same \cite{press2017using}. In addition to the weight-tying, we apply orthogonalization to the output embedding with the linearly independent conditions so that the output embedding only interacts with the corresponding token.  This assumption also requires that the token embeddings are over-parameterized and $d\geq K$. Under this assumption, applying Lemma \ref{lemma identity}, we find that $\Cb f_{\W}(X)=\bpi^X$. Thus, sampling from the classifier output $\Cb f_{\W}(X)$ becomes equivalent to \emph{sampling-from-softmax}. While the assumption is rather strong, it will enable us to develop a thorough theoretical understanding of self-attention through the \PCMC connection and convex log-likelihood formulation that facilitates consistency and sample complexity analysis.






\textbf{Bijection between attention and \PCMC dynamics.} We will next show that, under Assumption \ref{assume iden} there is a bijection between weights $\W$ and stochastic matrices $\Pb$. This will be established over a $(K-1)\times K$ subspace of $d\times d$ which is the degrees of freedom of the stochastic matrices.
\begin{definition}\label{defn Stoken}
    Let $\Sall \subset \R^{d \times d}$ be the subspace spanned by the matrices in $\{(\eb_i - \eb_j) \eb_k^\top :~i,j,k \in [K]\}$. 
\end{definition}

The next lemma states that the projection of $\W$ orthogonal to $\Sall$ has no impact on the model output and justifies the definition of $\Sall$.
\begin{lemma}\label{lemma Stoken}
For all $\W \in \R^{d \times d}$ and $X$, we have $f_{\W}(X) = f_{\bPi_{\Sall}(\W)}(X)$.

\end{lemma}
The proof of this lemma is provided in Appendix \ref{lemma app Stoken}. With this, we are ready to establish the equivalency between the \PCMC and self-attention dynamics.
\begin{figure*}[tb] 
   \centering    \includegraphics[width=0.9\textwidth]{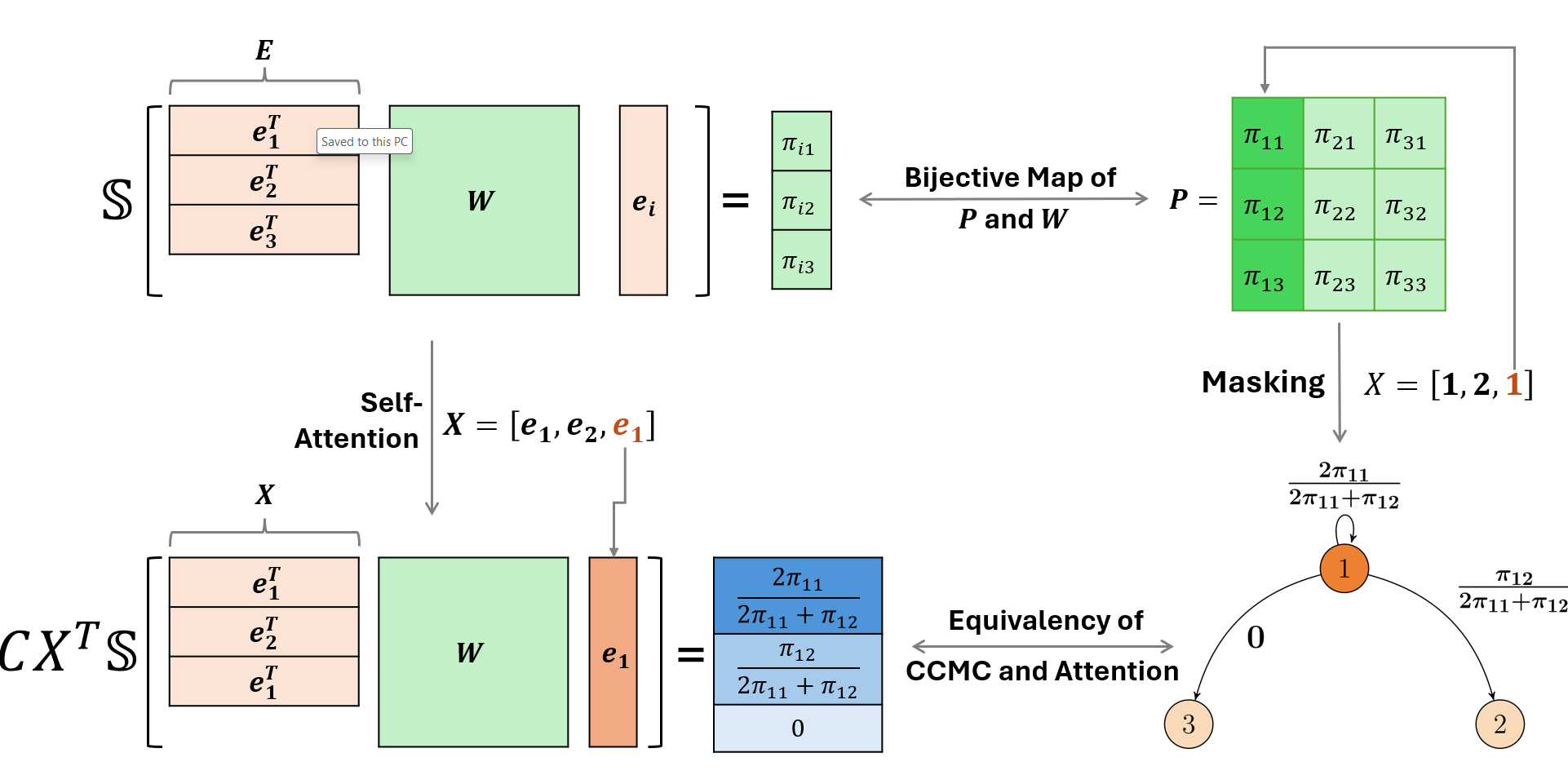}
   \vspace{-10pt}
   \caption{Illustration of the Equivalency between the Attention and PCMC models. We provide an example where the vocabulary size $K=3$ and the input prompt is $X = [1,2, 1]$. The upper figure represents how the token probabilities $\sft{\Eb\W\eb_i}$ can be mapped to a base transition matrix $\Pb$. The left-lower figure demonstrates the output of the self-attention given an input prompt $\X$. The right-lower figure derives CCMC transitions from this $\Pb$ given the same prompt. The resulting next token probabilities are the same for both of the models. The masking operation is demonstrated in a more detailed way in Figure \ref{fig:main_figure2}.
   }
       \vspace{-10pt}
   \label{fig:main_figure}
\end{figure*}




\begin{theorem}\label{theorem equivalency ps}
    Suppose Assumption \ref{assume iden} holds. For any transition matrix $\Pb$ with non-zero entries, there is a unique $\W\in \Sall$ such that for any prompt $X\in[K]^L$ and next token $y=x_{L+1}\in[K]$,
    \begin{align*}
        \P_{\Pb}(y | X ) = \cb_{y}^\top \X^\top \sft{\X \W \xl}
    \end{align*}
    where $\cb_y$ is the $y$-th row of the linear prediction head $\Cb$.
\end{theorem}

The proof of this theorem is provided in Appendix \ref{theorem app equivalency ps}. Note that for any $\W \in \R^{d \times d}$, there exists a transition matrix that satisfies the above by Lemma \ref{lemma identity}. As a result, Theorem \ref{theorem equivalency ps} establishes the equivalence between the \PCMC model and the self-attention model in the sense that $\P_{\Pb}(y| X)$ matches the output distribution of the self-attention model for any input prompt $X$. We will utilize this for learning latent attention models in Sections \ref{sect consistency} and \ref{sect single traj}. 

\subsection{Cross-Attention Model}\label{sec cross}
We also consider the cross-attention model where the query token $\xl$ is not among the keys for the attention head. Thus, $\xl$ becomes a free variable resulting in a more flexible transition model. The cross-attention is given by
\begin{align} 
\label{eq:cattn-def}
&\text{Key tokens:}~\bar{X}=[x_1~\dots~x_{L-1}] \nonumber \\
&f^{\texttt{CA}}_{\W}(X)=\Xb^\top \sft{\Xb\W\xl}\in\R^d.\tag{CA}
\end{align}

The \PCMC associated with the cross-attention only slightly differs from Definition \ref{def pcmc}. Now, the transition probabilities are defined with $\maskb=\Fc(\bar{X})$ because the model can only sample from the states/keys contained in $\bar{X}$. This clearly disentangles the \emph{state} $\xl$ of the \PCMC and the transition weighting vector $\Fc(\bar{X})$. Specifically, unlike self-attention, this \PCMC is not biased towards transitioning to the last token $\xl$ and we can entirely mask out last token in the next transition (as soon as $\xl$ is not contained within $\bar{X}$).
\section{Consistent Estimation: When can we learn an attention layer by prompting it?}\label{sect consistency}

In this section, our interest is learning a ground-truth attention layer by sampling (prompt, next-token) pairs. Let $\Dcx$ denote the distribution of input prompts. We assume $\Dcx$ has a finite support and denote its support by $\Omega$, which is a set of prompts. Finally, let $\Wgrd$ denote the ground-truth attention weights which will be our generative model. Specifically, we will sample the next token $y\in[K]$ according to the model output $\Cb f_{\Wgrd}(X)$ under Assumption \ref{assume iden}. Let $\Dcxy$ be this distribution of $(X, y)$ such that $X\sim\Dcx$ and $y\bgl X$ is sampled from $\Cb f_{\Wgrd}(X)$.

Throughout the section, we identify the conditions on the input prompt distribution $\Dcx$ and $\Wgrd$ that guarantee consistent learning of the attention matrix $\Wgrd$ in the population limit (given infinitely many samples from $\Dcxy$). We will investigate the maximum likelihood estimation procedure which is obtained by minimizing the negative log-likelihood. As a result, our estimation procedure corresponds to the following optimization problem.
\begin{align}
&\noindent\Wst = \arg \min_{\W \in \Sall} \Lc(\W)\quad \text{where} \label{defn Wstr}\\
 &\Lc(\W) = \E_{(X,y) \sim \Dcxy}[-\log(\cb_{y}^\top \X^\top \sft{\X\W\xl})]. \nonumber 
\end{align}
Note that the prompt length is not necessarily the same. Here we use $\x_L$ to simplify the notation and it presents the last/query token of prompt $\X$.
\begin{definition}\label{defn consistency of estimation}
The estimator \eqref{defn Wstr} is \textbf{consistent} if $\Wst=\Wgrd$ when data $\Dcxy$ is sampled from $\Wgrd$. Otherwise, the estimation is called inconsistent.
\end{definition}





Observe that the optimization in \eqref{defn Wstr} is performed over the subspace $\Sall$. As discussed in Lemma \ref{lemma Stoken}, its orthogonal complement $\Sall^\perp$ has no impact on the output of the attention model. Similarly, the gradient of $\Lc(\W)$ is zero over $\Sall^\perp$, thus $\Sall^\perp$ is essentially the \emph{null space} of the problem where token embeddings simply don't interact. For this reason, we make the following assumption.
\begin{assumption}\label{assume stoken consistency}
    The ground truth obeys $\Wgrd\in\Sall$.
\end{assumption}

Learning the ground-truth model is challenging for two reasons: First, through each prompt, we only collect partial observations of the underlying model. It is not clear if these observations can be \emph{stitched} to recover the full model. For instance, if we query an LLM on a subset of domains (e.g.,~only query about medicine and law), we might not be able to deduce its behavior on another domain (e.g.,~computer science, math). Second, optimization of self-attention is typically non-convex \cite{tarzanagh2023transformers} and does not result in global convergence. Fortunately, under Assumption \ref{assume iden}, \eqref{defn Wstr} becomes a convex problem \cite{anonymous} which we will leverage in our analysis at the end of this section.

Section \ref{sect setup} established the equivalency between the \PCMC model and the self-attention model under Assumption \ref{assume stoken consistency}. This enables us to prove the consistency of estimation for the self-attention model through the consistency of estimation for the \PCMC model. We define the population estimate of the transition matrix for the \PCMC model as follows:
\begin{align}
            &\Pst = \arg \min_{\Pb} \Lc(\Pb)\quad\text{where} \label{defn Pstr}\\
    &\Lc(\Pb) = \E_{(X,y) \sim \Dcxy}[-\log(\P_{\Pb}(y | X))]. \nonumber
\end{align}
Through Theorem \ref{theorem equivalency ps}, there is a mapping between the optimal solution set of \eqref{defn Wstr} and \eqref{defn Pstr}.

The consistency conditions on the input prompt distribution and the ground truth variable $\Wgrd$, or equivalently $\Pgrd \triangleq \Pb^{\Wgrd},$ are related to \emph{how well input prompts $\Omega$ cover the pairwise token relations}. To characterize this, for each last token (i.e.,~query token), we define an undirected co-occurrence graph as follows.
\begin{definition}\label{defn cooccurrence graph}
     Let $\Omega_k\subset\Omega$ be the set of input prompts whose last tokens are equal to $k$ for all $k\in [K]$. Define the co-occurrence graph $\Gc^{(k)}$ with $K$ vertices as follows: There is an edge between $i$-th and $j$-th nodes in $\Gc^{(k)}$ iff there is a prompt $X\in\Omega_k$ such that its \emph{key tokens} include both $i$ and $j$. Here, \emph{key tokens} are all tokens for self-attention and all but last tokens for cross-attention (i.e.,~$\bar{X}$ in \eqref{eq:cattn-def}).
\end{definition}
A simple example highlighting the difference between self-attention and cross-attention is given in Figure~\ref{fig:consistency_intro}. Based on this graph, the consistency of estimation is then translated to the connectivity of the graphs $\Gc^{(k)}$. This is summarized in the following theorem.

\begin{figure*}[tb] 
    \centering    \includegraphics[width=0.9\textwidth]{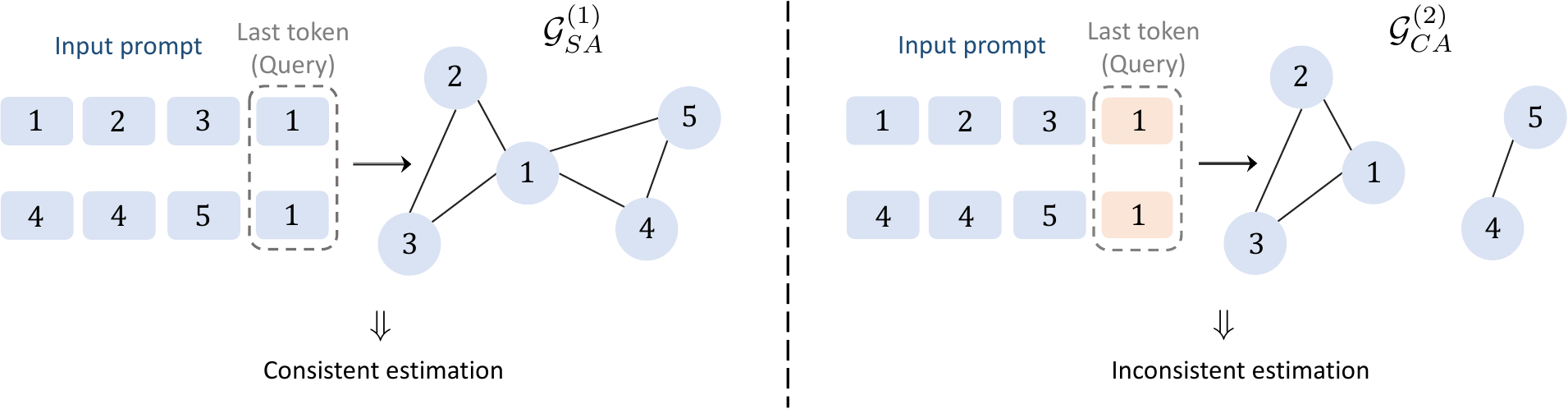}
    \vspace{-10pt}
    \caption{\small{{Illustration of co-occurrence graphs for the self-attention (\textbf{Left}) and cross-attention (\textbf{Right}) models}. We fit the same input examples where the only difference is the use of self- vs cross-attention (which includes vs excludes the query token `$1$' from the list of key tokens). Following Theorem~\ref{theorem consistency}, in the cross-attention setting, where the token `$1$' is not contained in the prompt, the co-occurrence graph becomes disconnected, resulting in inconsistent estimation. In contrast, the estimation is consistent for the self-attention model since both inputs share the same query token within their key tokens. 
    }}
        \vspace{-10pt}
    \label{fig:consistency_intro}
\end{figure*}
\begin{theorem}\label{theorem consistency}
    Let $\Pgrd$ be a transition matrix with non-zero entries. Let $(\Gc^{(k)})_{k=1}^K$ be the co-occurrence graphs based on the input distribution $\Dcx$. Then, the estimation $\Pst$ in \eqref{defn Pstr} is consistent iff $\Gc^{(k)}$ is a connected graph for all $k \in [K]$. 
\end{theorem}
The proof of Theorem \ref{theorem consistency} is provided in Appendix \ref{app sect consistency}. This appendix also addresses the case where $\Pgrd$ contains zero transition probabilities. 
The main proof idea is that, optimizing the log-likelihood of a specific prompt results in estimating the local Markov chain transitions over the tokens within that prompt. When two prompts are connected and optimized together, the optimization merges and expands these local chains.

The connectivity of the graph is essentially a \emph{coverage condition} that ensures that prompts can fully sense the underlying Markov chain. For the self-attention model, this condition simplifies quite a bit.
\begin{observation} For self-attention model, $\Gc^{(k)}$ is connected iff all tokens in $[K]$ appear at least in one prompt within $\Omega_k$.
\end{observation}
This observation follows from the fact that, for self-attention, the last token is always within the keys of the input prompt. Thus, all prompts within $\Omega_k$ overlap at the last token $\xxl=k$. Thus, the node $k$ naturally connects all tokens that appear within the prompts (at least once). The situation is more intricate for cross-attention -- where the query token is distinct from keys -- because there is no immediate connectivity between prompts. The difference between these two models are illustrated in Figure~\ref{fig:consistency_intro}. For this reason, our theory on cross-attention is strictly more general than self-attention and its connection to Markov chain is of broader interest.


Our next result is an application of Theorems \ref{theorem equivalency ps} and \ref{theorem consistency} and states the result for self-attention problem \eqref{defn Wstr}. It equivalently applies to the cross-attention variation of \eqref{defn Wstr}.


\begin{corollary}\label{corollary consistency}
    Let $\Wgrd$ be the attention model underlying $\Dcxy$ for either the self-attention model or the cross-attention model and suppose Assumptions \ref{assume iden} and \ref{assume stoken consistency} hold. Then, $\Wgrd=\Wst$ in \eqref{defn Wstr} iff all $\Gc^{(k)}$'s are connected.
\end{corollary}

\subsection{Gradient-based Optimization of Attention Weights}
An important feature of our problem formulation \eqref{defn Wstr} is its convexity, which was observed by \cite{anonymous}. The convexity arises from the fact that, we directly feed the attention probabilities to log-likelihood which results in a LogSumExp function. We next state the following stronger lemma which follows from Lemma 9 of \cite{anonymous}). A detailed discussion is provided in Appendix \ref{sect app strict convexity}. 
\begin{lemma}\label{lemma strict convexity consistency}
Suppose Assumption \ref{assume iden} holds. If $\Gc^{(k)}$ is a connected graph for every $k \in [K]$, then $\Lc(\W)$ of \eqref{defn Wstr} is strictly convex over $\Sall$.
\end{lemma}
For any continuous strictly convex function, there exists at most $1$ solution that minimizes the objective function. Furthermore, by the definition of $\Dcxy$, we know that $\Wgrd$ is a solution. Therefore, Lemma \ref{lemma strict convexity consistency} guarantees the gradient-based learnability of the attention model as follows.
\begin{corollary} Set $\W_0=0$ and run gradient iterations $\W_{t+1}\gets \W_t-\eta \nabla\Lc(\W_t)$ for $t\geq 0$ with a constant $\eta$. If all $\Gc^{(k)}$'s are connected, then $\lim_{t\rightarrow\infty}\W_t=\Wgrd$.
\end{corollary}

\section{Guarantees on Finite Sample Learning}\label{sect finite sample}

\begin{figure}[tb] 
    \centering    
    \begin{minipage}{1\linewidth}
\vspace{-7pt}
\includegraphics[width=1\linewidth]{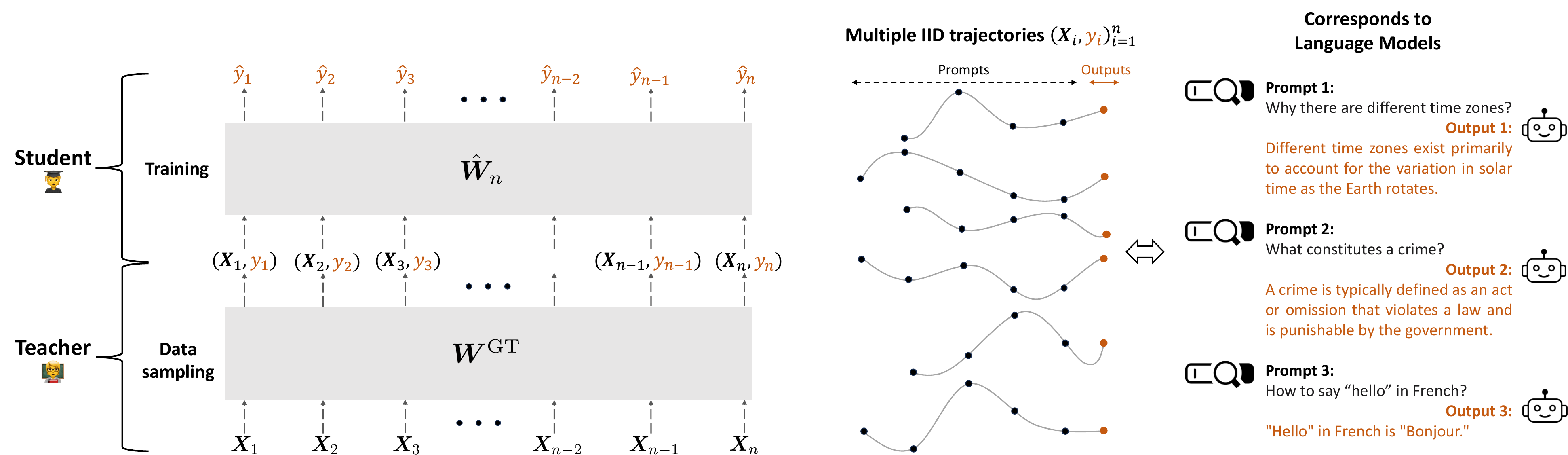}
    \centering
    \label{fig:iid-learning}
\end{minipage}
    \centering
\caption{\small{\textbf{Left: }Illustration of finite sample learning where the next tokens are sampled from the ground-truth model, which corresponds to single outputs from multiple IID trajectories. \textbf{Right: }In practice, the scenario is analogous to querying language models with prompts on different topics and using the responses to train a tiny model. In Theorem~\ref{theorem consistency}, we characterize the condition when the tiny model can estimate the ground-truth model consistently.}}
\vspace{-5pt}
\end{figure}
Following the setting of Section~\ref{sect consistency}, we sample a training dataset $\data=\big\{(X_i, y_i)\big\}_{i=1}^n$ from $\Dcxy$. In this section, we establish a sample complexity guarantee on the difference between $\tf{\Wh - \Wgrd}$ where $\Wh$ is trained on $\data$. 

\textbf{ERM problem.} Given a training dataset $\data$, we consider the ERM problem with the following objective:
\begin{align}
&\Wh_n = \arg \min_{\W \in \Sall } \Lch_n(\W)\quad\text{where } \label{Whn} \\ &\Lch_n(\W)=\frac{1}{n}\sum_{i=1}^n -\log(\cb_{y_i}^\top \X_i^\top \sft{\X_i\W\xli}). \label{EmpRisk W}
\end{align}
Our main aim in this section is to establish a sample complexity guarantee on $\tf{\Wh - \Wgrd}$. We leverage the findings of Section \ref{sect consistency} to achieve our aim with the following assumption:
\begin{assumption}\label{assume cooccurrence connected}
    Recall the co-occurrence graphs in Definition \ref{defn cooccurrence graph}. We assume that the co-occurrence graphs $(\Gc^{(k)})_{k=1}^K$ constructed from $\Dcx$ are connected.
\end{assumption}
We prove the following theorem, which will provide finite sample complexity guarantees for the loss function:



\begin{theorem}\label{theorem sample complexity}
    Suppose Assumptions \ref{assume iden} and \ref{assume cooccurrence connected} hold. Let $R_0>0$ be a finite constant based on the structure of $\Wgrd$ and $\Dcx$. Then, if $n \geq R_0 K^2$, with probability at least $1-2\delta$
    \begin{align*}
        \Lc(\Wh_n) - \Lc(\Wst) \lesssim \frac{K^2\log\frac{n}{K\delta}}{n}.
    \end{align*}
\end{theorem}
The proof is provided in Appendix \ref{sect app sample complexity}. We apply the local covering arguments to achieve sample complexity guarantees by concentration inequalities. Using Lemma \ref{lemma strict convexity consistency}, we prove that $\Wh_n$ inside a local ball with sufficient samples and we achieve the fast rate $1/n$ with the smoothness of the loss function similar to \cite{Bartlett_2005} and \cite{srebro_FastRate}.

Now, we are ready to share our main contribution in this section with the following corollary:
\begin{corollary}\label{corollary sample complexity}
    Consider the setting in Theorem \ref{theorem sample complexity} and suppose Assumptions \ref{assume iden}, \ref{assume stoken consistency}, and \ref{assume cooccurrence connected} hold. Then, if $n \geq R_0 K^2$, with probability at least $1-2\delta$
    \begin{align}
        \tf{\Wh_n - \Wgrd}^2 \lesssim \frac{K^2\log\frac{n}{K\delta}}{n}.
    \end{align}
\end{corollary}

The proof of this corollary is provided in Appendix \ref{corollary sample complexity app} and follows from Lemma \ref{lemma strict convexity consistency} and Theorem \ref{theorem sample complexity}. 
\section{Learning from a Single Trajectory}\label{sect single traj}

In this section, we ask: Can we learn a ground-truth self-attention model by querying it once and collecting its output trajectory? The question of single-trajectory learning is fundamental for two reasons: First, modern language models train all tokens in parallel, that is we fit multiple next tokens per sequence. Secondly, learning from a single trajectory is inherently challenging due to dependent data and has been subject of intense research in reinforcement learning and control. Here, we initiate the statistical study of single trajectory learning for attention models.

\noindent\textbf{Setting:} We first describe the single trajectory sampling. Suppose we are given a ground truth attention matrix $\Wgrd$ and an initial prompt $X_1$ of length $L$. We feed $X_1$ to sample the next token $y_1$ and auto-regressively feed back each next token to sample a length $n$ output trajectory. Overall, we obtain the training dataset $\data=\big\{(X_{i}, y_{i})\}_{i=1}^n$ where $X_i=[X_1~y_1~\dots~y_{i-1}]$. This sampling is done according to our self-attention model under Assumption \ref{assume iden}. To proceed, we optimize the likelihood according to \eqref{Whn} and obtain $\Wh_n$.




The consistency of estimation for the single trajectory sampling is defined as follows.
\begin{definition}
    Recall the estimation $\Wh_n$ of $\Wgrd$ in \eqref{Whn}. The estimation $\Wh_n$ is called consistent if and only if 
    \begin{align*}
        \P\left( \lim_{n \to \infty} \Wh_n = \Wgrd\right) = 1.
    \end{align*}
\end{definition}
Let $\Tc\subset[K]$ be the set of tokens that appear in $X_1$. Recall that, our attention model samples the next token from the input prompt, thus, all generated tokens will be within $\Tc$. Thus, from a single trajectory, we can at most learn the local Markov chain induced by $\Tc$ and consistency is only possible if $X_1$ contains all $[K]$ tokens. Thus, we assume $\Tc=[K]$ and $X_1$ contains the full vocabulary going forward.


In what follows, we study two critical behaviors:
\begin{itemize}[leftmargin=4mm, itemsep=1mm, partopsep=0pt,parsep=0pt]
\item \textbf{(Q1)} How does the distribution of generated tokens $y_i$ evolve as a function of the time index $i$?
\item \textbf{(Q2)} When is consistent estimation of the underlying attention model $\Wgrd$ possible?
\end{itemize}

\subsection{Empirical Investigation}\label{sec:emp-single}
\begin{figure*}[!tb]
\vspace{-5pt}
\centering
\begin{minipage}{0.42\linewidth}
\includegraphics[width=.9\linewidth]{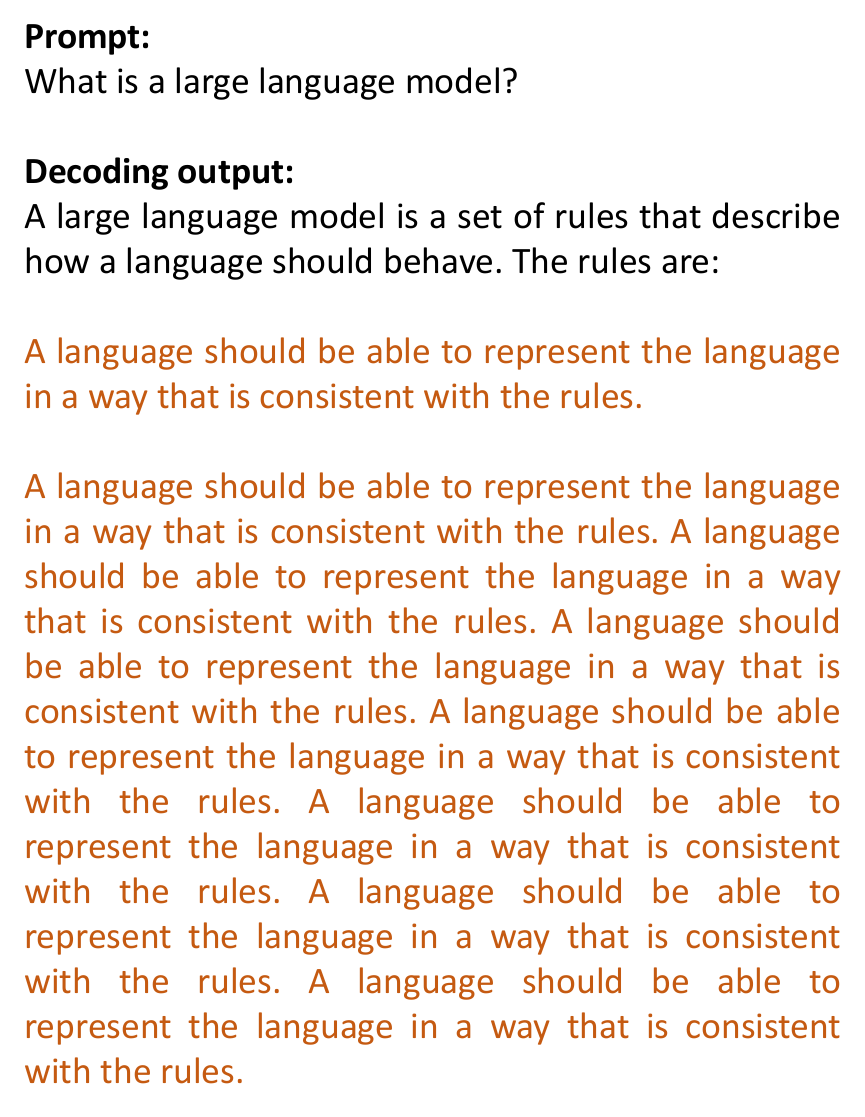}
\vspace{-7pt}
\end{minipage}
\rulesep
\begin{minipage}{0.55\linewidth}
\includegraphics[width=.9\linewidth]{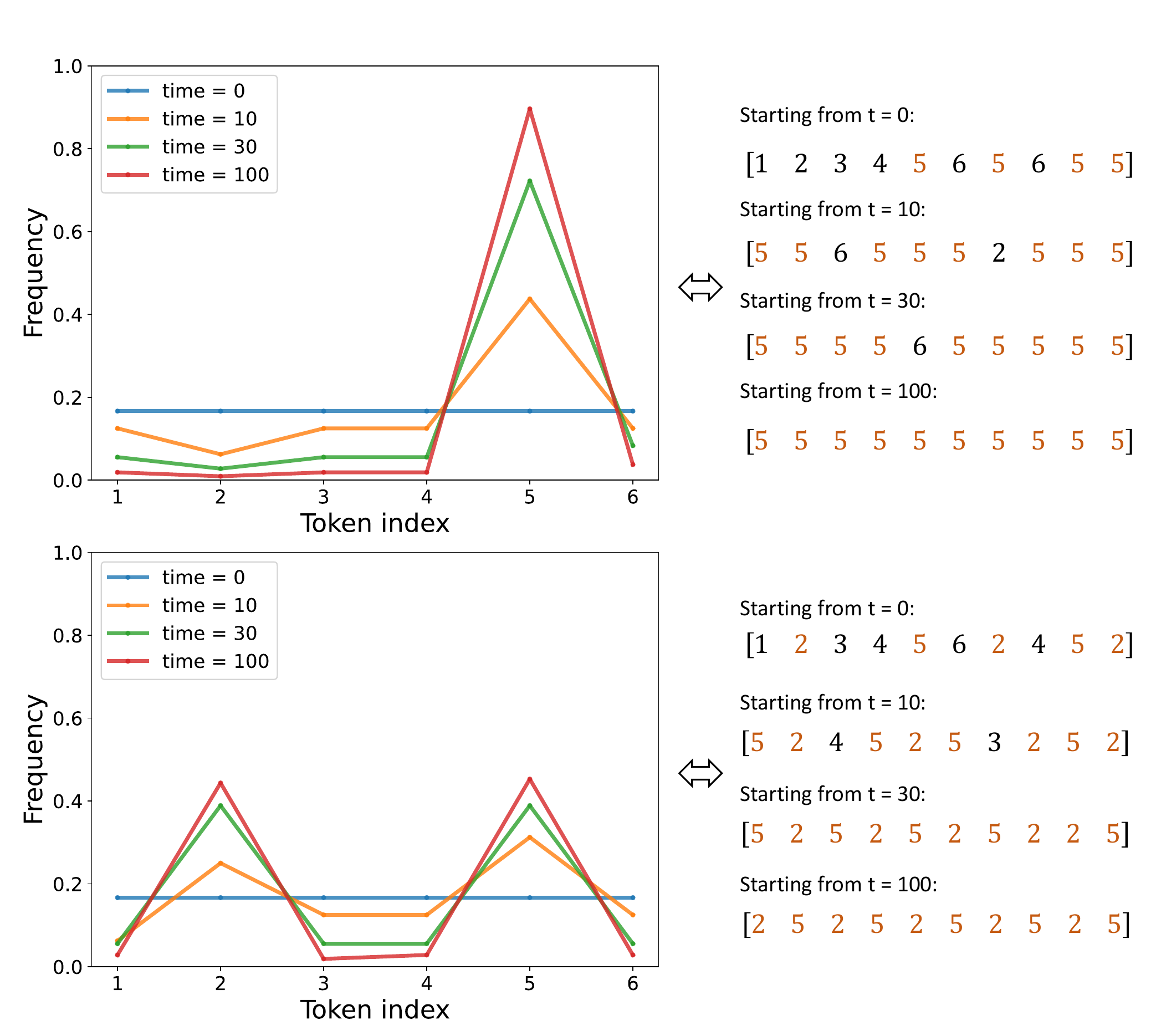}
\vspace{-7pt}
\end{minipage}
\caption{\small{Demonstration of Distribution Collapse/Repetition. 
\textbf{Left:} An example query where the GPT-2 response quickly degenerates into repetition. \textbf{Middle:} We generate two single self-attention trajectories with a vocabulary containing $K=6$ tokens and plot the empirical token frequencies (for each token in the vocabulary). The upper figure is generated using a randomly initialized transition matrix, while the lower one is generated using the same transition matrix as the upper one except that the diagonal entries are set to $0$, enforcing that the probability of query token $i \to \text{ next token }i$ is $0$. The frequency is calculated as the ratio of token occurrences to the sequence length at that time. \textbf{Right:} Trajectory snapshots with a 10-token window from time index $i$ revealing that token $5$ (upper) / tokens $2 \text{ and } 5$ (lower) dominate the trajectory. The lower right is dominated by two tokens because a single token cannot self-reinforce due to zero diagonals. }}\label{fig token freq}
\vspace{-10pt}
\end{figure*}
We first describe our experiments which elucidate why these questions are interesting and provide a strong motivation for the theory. Note that we have an equivalency between the attention models and the \PCMC model, which is true for any sampling method. Throughout this section, we discuss the consistency of $\Pgrd$, which directly implies the consistency of $\Wgrd$.

\textbf{The \emph{Distribution Collapse} phenomenon.}
To gain more motivation, we randomly initialize a one-layer self-attention model and generate a single trajectory with a length of $500$ starting from initial prompt $X_1 = [6]$. We track the evolution of token frequency as shown in the middle of Figure~\ref{fig token freq}. As illustrated in the right side of Figure~\ref{fig token freq}, when the generation time increases, the diversity of output tokens is greatly reduced, which eventually collapses to a singleton. This arises due to the \emph{self-reinforcement} of majority tokens within the trajectory. This phenomenon also corresponds to the repetition problem found in text generation by language models, as demonstrated in the left side of Figure~\ref{fig token freq}. 

\subsection{Theoretical Study of Single Trajectory Sampling}
To learn the underlying self-attention model from a single trajectory, we must visit each token/state infinitely many times. Otherwise, we could not learn the probability transitions from that last token choice. Let $S_{k, n}$ be the number of occurrences of token $k$ within $X_n$. Our first result shows that each token is guaranteed to be visited infinitely many times as the trajectory grows.
\begin{lemma} \label{lem inf token visit} Let $\Pgrd$ be a transition matrix with non-zero entries. We have that $\Pro(\lim_{n\rightarrow \infty}S_{k,n}=\infty)=1$ for all $k\in[K]$.
\end{lemma}
The proof of this lemma can be found in Appendix~\ref{app sect single traj} and follows from an application of the Borel-Cantelli lemma.

This lemma is insightful in light of Figure~\ref{fig token freq}: Even if the distribution of the generated tokens collapses to a singleton, any token within the input prompt will keep appearing albeit with potentially vanishing frequencies. Next, we discuss the condition when the distribution collapse will happen in response to Figure~\ref{fig token freq}. To characterize the condition for the distribution collapse, we consider a \PCMC model with $K = 2$. Suppose the ground-truth transition matrix $\Pgrd = \begin{bmatrix} 
    1-p & 1- p \\ 
    p & p
\end{bmatrix}$. Without loss of generality, assume $p \leq 1/2$. For brevity, we define the \emph{weak token} as the token with a smaller transition probability $p$, which is Token $2$ in our setting. The next result bounds the frequency of the weak token as the trajectory grows.
\begin{lemma}[Distribution collapse]\label{lem dist col}
Consider the \PCMC model with $K = 2$ defined in Section \ref{sec:emp-single}. Suppose that $\X_1$ includes all vocabulary at least once. Recall that $\Fc(X_t)$ denotes the empirical frequency of individual states where $X_t$ is the state trajectory at time $t$. {For any $t > t_0$ with a sufficiently large $t_0$}, we have:
    \begin{align*}
        \E[\m(X_t)_2] <  t^{-q}
    \end{align*}
    where $q = 1 - p / (1-p)$. Furthermore, when $p < 1/2$, 
    \begin{align*}
        \lim_{t \to \infty} \E\left[\frac{\Fc(X_t)_2}{\Fc(X_t)_1}\right]=0.
    \end{align*}
\end{lemma}
 
\begin{wrapfigure}{R}{0.5\textwidth}   
  \begin{center}
\includegraphics[width=0.48\textwidth]{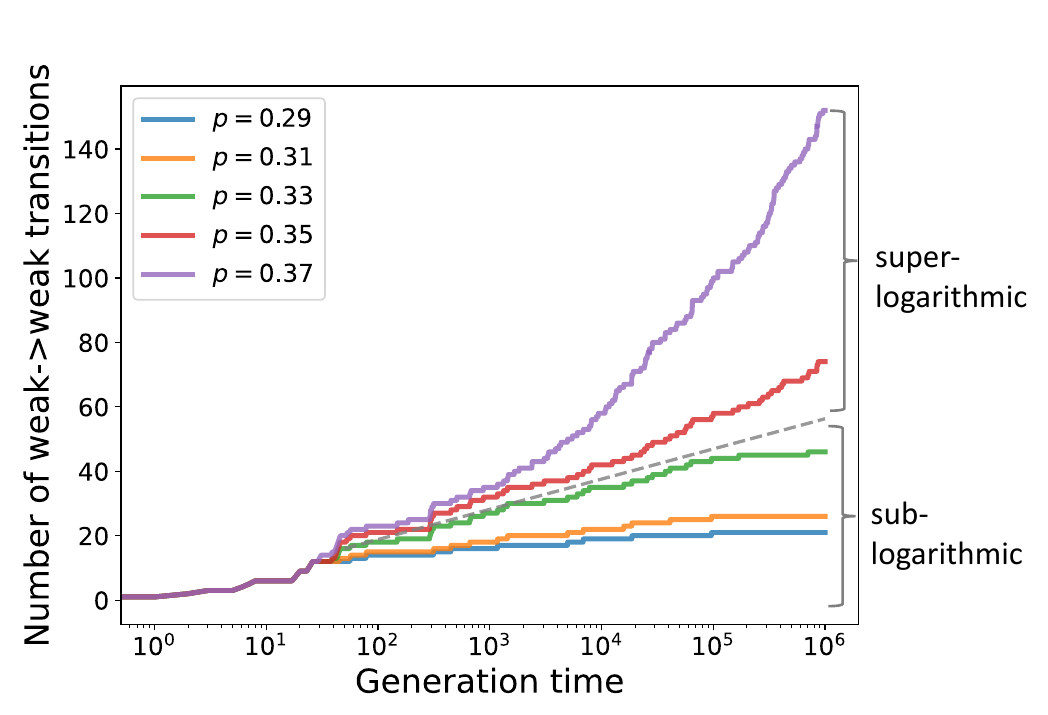}
  \end{center}
  \caption{Growth of weak to weak transition}
  \label{fig:weak-trans}
\end{wrapfigure}

The proof is provided in Appendix~\ref{app sect single traj}. Lemma~\ref{lem dist col} states that token $1$ will dominate the trajectory when $p < 1/2$. As a result, the distribution collapse will happen as long as $p \neq 1/2$ due to the symmetry. 
So far, we show that all tokens are visited infinitely many but their frequencies can vanish. On the other hand, to learn the underlying transition matrix, we should also visit each transition (from state/token $i$ to state/token $j$) infinitely many times (otherwise the estimated $p$ would be 0). To study this question, we ask: \emph{Will self-attention visit the transition from the weak token to itself forever?} Figure \ref{fig:weak-trans} shows the number of weak$\rightarrow$weak transitions for varying $p$ choices. We observe that this number grows super-logarithmic in trajectory length when $p$ exceeds $1/3$ and it is sub-logarithmic when $p$ is smaller than $1/3$. We argue that this sub-logarithmic (very slow) growth is actually an indicator of the fact that there are actually finitely many weak$\rightarrow$weak throughout the trajectory, which would in turn make estimation of the second column of $\Pgrd$ inconsistent.


To justify this, we utilize our theory to study the growth of weak$\rightarrow$weak transitions (albeit non-rigorously). Lemma \ref{lem dist col} shows that the expected density of the weak token is $t^{-q}$ throughout the trajectory. Let us treat this expectation as the true weak token probability at time $t$. Next, since the trajectory contains only $O(t^{-q})$ fraction weak tokens, due to the \PCMC model, the chance of transition to a weak token (from any token) is $O(t^{-q})$. Combining these, we find that $\Pro(\text{weak}\rightarrow\text{weak})=\Pro(\text{weak}|\text{weak})\Pro(\text{weak})\propto t^{-2q}$. With this estimate at hand, we can use Borel-Cantelli to study finiteness of weak$\rightarrow$weak transitions. Specifically $\int_{t=1}^\infty t^{-2q}$ is finite when $q\geq 1/2$ and infinite when $q<1/2$. This translates to $p\leq 1/3$ and $p>1/3$ respectively and remarkably coincides with the sub/super-logarithmic growth observed in Figure \ref{fig:weak-trans}.




\begin{figure}[tb] 
    \centering    
    \begin{minipage}{0.9\linewidth}
\vspace{-7pt}
\includegraphics[width=0.95\linewidth]{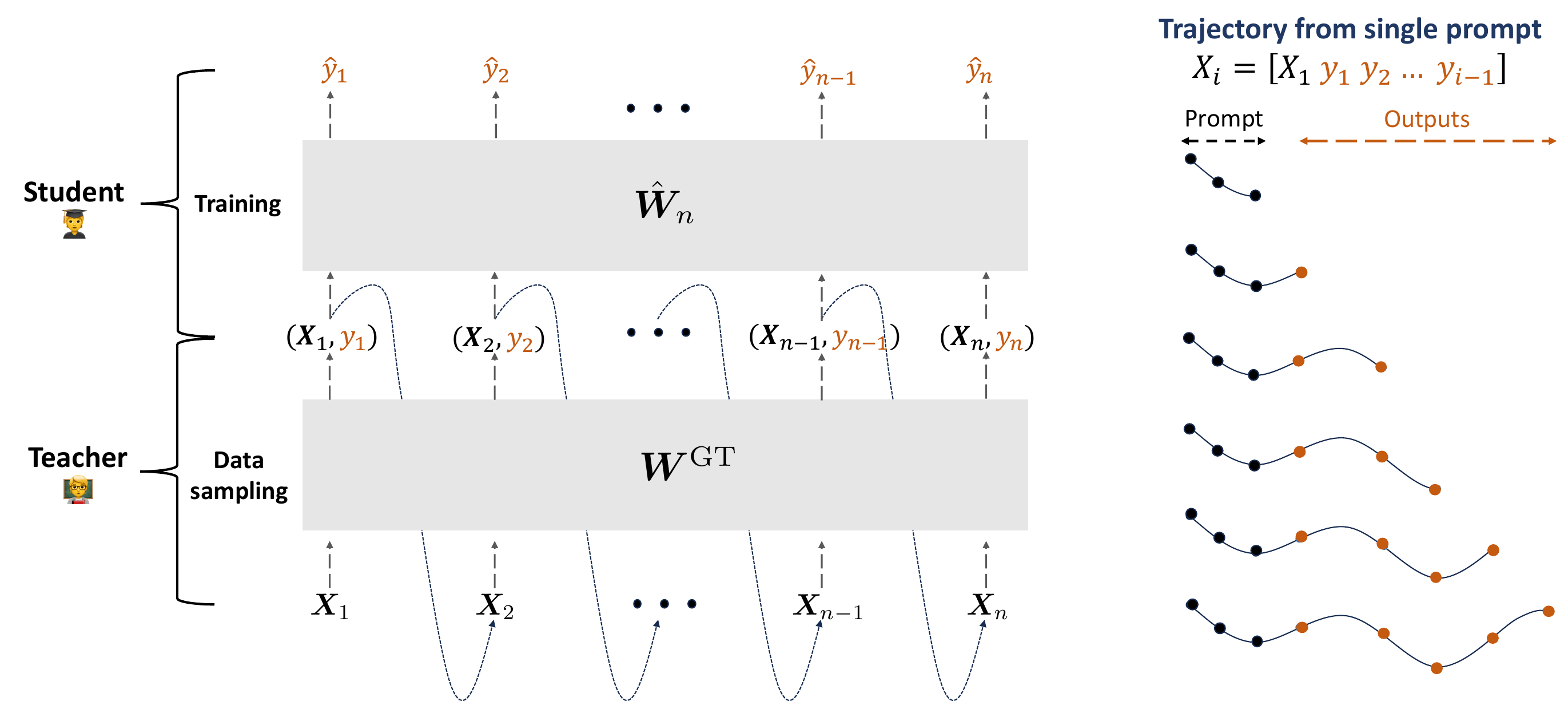}
    \vspace{-6pt}
    \centering
    \caption{\small{Illustration of single-trajectory learning where next tokens are sampled from a single trajectory. The setting is analogous to asking language models a broad question and constantly collecting the responses.
    }}
    \vspace{-12pt}
    \label{fig:single-learning}
\end{minipage}
\end{figure}

\section{The Role of Positional Encoding}\label{sect position embedding}
In Section \ref{sect setup}, we have established an equivalence between the attention models and the \PCMC model using a mask. This mask was defined by $\Fc(X)$ in \eqref{eq:freq-func}, which is associated with the occurrences of different tokens in an input prompt. In this section, we incorporate positional encodings into the \PCMC model to study its impact on the transition dynamics.

To proceed, suppose the input prompt is fixed to be length $L$. We will use absolute position encoding \citet{vaswani2017} which adds a position vector $\ub_i$ at position $i$ for $1\leq i\leq L$. Recalling $\eb_{x_i}$ is the vocabulary embedding, this leads to the following input embedding $\x_i = \ub_i + \eb_{x_i}$.
Let $\Ub := [\ub_1~\dots~\ub_L]^\top \in \R^{L \times d}$ be the positional embedding matrix with $\ub_i \in \R^d$. The linear classifier $\Cb$ in the attention models predicts the next token ID. In addition to Lemma \ref{assume iden}, we assume the following to ensure that there is no bias to the classifier output from position embeddings.
\begin{assumption}\label{assume positional embedding C}
    The projection of positional embedding onto the columns of $\Cb$ is zero, i.e., $\Cb \Ub^\top = \mathbf{0}$
\end{assumption}
To quantify the effect of positional embedding on the output of the attention model, we define the variables $\ab \in \R^L, \bb \in \R^K$, and $\Vb \in \R^{L \times K}$ as follows: 
\begin{equation*}
    \begin{split}
        \ab \coloneqq \exp(\Ub\W\ub_{L}) \quad & \quad
        \bb \coloneqq \exp(\Eb\W\ub_{L}) \\ 
        \Vb \coloneqq \exp&(\Ub\W\Eb^{\top}) 
    \end{split}
\end{equation*}
where $\exp(\cdot)$ represents the element-wise exponential function. Then, we define the probability distribution characterizing the \PCMC model as follows: $\P_{(\Pb,\Ub)}(x_{L+1} = j | X) =$
\begin{equation}\label{pos PCMC transition}
\begin{split}
     \hspace{-1cm} \frac{b_j \pi_{\xxl, j} \sum_{i = 1}^L a_i  V_{i,x_L} \cdot \boldsymbol{1}({x_i = j}) }{\sum_{k = 1}^K b_k \pi_{\xxl, k} \sum_{i = 1}^L a_i  V_{i,x_L} \cdot \boldsymbol{1}({x_i = k}) } 
\end{split}
\end{equation}
where $\pi_{ij}$ is based on $\Pb^{\W}$, defined in \eqref{pbw def}. The intuition behind \eqref{pos PCMC transition} is that, in Section \ref{sect setup}, the \PCMC model is constructed by a mask with $K-$dimension whereas $\eqref{pos PCMC transition}$ can be considered as a mask with $K\times L$ dimension. Now, we are ready to share our main results of this section:
\begin{lemma}\label{lemma positional embedding}
    Suppose Assumptions \ref{assume iden} and \ref{assume positional embedding C} hold. Then, for any $\W \in \R^{d \times d}$, there exists the transition matrix $\Pb^{\W}$ such that for any $(X, y)$ we have the following:
    \begin{align*}
        \P_{\Pb, \Ub}( y | X) = \cb_{y}^\top \X^\top \sft{\X \W \xl}
    \end{align*}
\end{lemma}
The proof is provided in Appendix~\ref{app sect position}. Note that the one-to-one map, consistency of estimation, and finite sample guarantee can be built upon Lemma~\ref{lemma positional embedding} for the \PCMC model with positional embedding by following similar arguments to the previous sections.  

\section{Related Work}\label{sec related}

\noindent\textbf{Theoretical treatment of attention models.}~
\citet{yun2020_universal, edelman2022_inductive, fu2023_randomfeatures, baldi2023_quarks} focused on expressive power or inductive biases of attention-based models. \cite{jelassi2022_transformer,li2023_transformer,oymak23a_prompt, tarzanagh2023max, tarzanagh2023transformers} studied optimization and generalization dynamics of simplified attention models for supervised classification tasks. \citet{tian2023_scan, anonymous} explored the training dynamics of such models for the next-token prediction task. To the best of our knowledge, we are the first ones to establish the connection between (context-conditioned) Markov chains and self-attention models and leverage that to establish rigorous guarantees for learning the models. Although not directly related to this work, there is also a growing body of literature on the theoretical study of in-context learning \cite{xie2022_incontext}; \cite{garg2022_incontext}; \cite{li2023_incontext}; \cite{akyurek2023_incontext}; \cite{oswald2023_incontext}.


\noindent \textbf{Learning Markov chains.}~The problem of estimating the transition matrix of a Markov chain from a single trajectory generated by the chain is a classical problem~\citep{billingsley1961statistical}. Recently, \citet{wolfer2019minimax,wolfer2021statistical,hao2018learning} (nearly) characterize minimax estimation risk for ergodic Markov chains. 
\citet{pmlr-v132-chan21a} consider the same problem for irreducible and reversible Markov chains. There is also a large literature on estimating the transition matrix 
under some structural constraints on the transition matrix such as low-rank assumption~\citep{zhang2019spectral, shah2020sample, stojanovic2023spectral, bi2023low, li2018estimation, zhu2022learning}. Interestingly, \citet{stojanovic2023spectral} also considers matrix estimation from multiple transitions observed from IID sampled states. Note that the text generation from attention models is not quite Markovian due to the context-dependent masking of the base Markov chain in \PCMC. Furthermore, we are interested in recovering the model weights $\W$ instead of the transition probabilities. A concurrent work \cite{makkuva2024attention} also explores the connection between a 1-layer attention model and Markov chains. The authors aim to learn a standard Markov chain using 1-layer self-attention whereas we show that self-attention is a non-Markovian model and we construct a general mapping between self-attention and a modified Markov chain dynamics. They primarily focus on the optimization geometry of the 1-layer attention model whereas we also provide learnability guarantees including consistency and sample complexity. On the other hand, our guarantees require a weight-tying condition (Assumption \ref{assume iden}) that convexifies our problem formulation (Lemma \ref{lemma cvx app}) whereas their analysis allow for a more general attention layer.

\noindent \textbf{Shortcomings in neural text generation.}~Multiple works~\citep[see, e.g.,][]{see2017get, holtzman2019curious, welleck2020consistency, xu2022learning} have explored various issues with the language model generated texts, especially repetitive nature of such texts. 
\citet{xu2022learning} argue that the self-enforcing behavior of language models leads to repetitions. This aligns with our formal analysis of self-attention models using \PCMC. Several studies offer training-based solutions to mitigate repetition \citep{xu2022learning, welleck2019neural, lin2021straight}, while others modify the decoding process~\citep{fan2018hierarchical,holtzman2019curious,welleck2020consistency}. Rather than proposing a new solution, our work rigorously characterizes the conditions under which repetition becomes inevitable in self-attention models.


\citet{fu2021theoretical} analyzed repetition in text generated by Markov models, attributing it to high-inflow words. Our work diverges significantly. Instead of assuming a Markovian process, we establish an equivalence between self-attention mechanisms and context-conditioned Markov chains. This leads to the non-Markovian generation, where the entire context influences the next token.  Furthermore, our theoretical analysis extends beyond repetition, using this \PCMC equivalence to explore the learnability of self-attention models from generated data. 

\textbf{Reinforcement learning and data-driven control.} The coverage condition that we have found for the consistency of estimation is also related to the data coverage condition in offline reinforcement learning  \citep{chen2019informationtheoretic_coverage, xie2020q_coverage, zhan2022offline_coverage, jin2022pessimism_coverage, foster2022offline_coverage, rashidinejad2023bridging_coverage}. In these works, given a dataset collected according to offline policies, we wish to learn optimal policy, which raises a distribution shift challenge. Their statistical analysis relies on the data coverage conditions during dataset collection to withstand the distribution shift. Our work is related to these at a high-level since we provide necessary and sufficient conditions on the input prompt distribution for the consistent estimation of a ground truth attention matrix.  Furthermore, we establish finite sample complexity guarantees under the coverage conditions that provide the consistency of estimation.  

Our work also relates to the literature on the statistical aspects of time-series prediction \citep{kuznetsov2014generalization,kuznetsov2016time,simchowitz2018learning,mohri2008rademacher} and learning (non)linear dynamics \citep{dean2020sample,ziemannlearning,dean2020sample,tsiamis2022statistical,sarkar2019near,sun2022finite,mania2020active,oymak2021revisiting,block2023smoothed}. Learning dynamical systems from a single trajectory has attracted significant attention in the recent literature \citep{ziemann2022single,oymak2019stochastic,sattar2022non,oymak2019non,matni2019tutorial,foster2020learning,ziemann2024sharp}. As long as the stochastic process is mixing (e.g.~ergodic Markov chain, stable dynamical system), the samples from the trajectory are approximately independent, and the underlying hypothesis can be learned under suitable assumptions \citep{yu1994rates}. There is also a recent emphasis on mitigating the need for mixing \cite{simchowitz2018learning,ziemann2022learning}. Unlike dynamical systems, MDPs, or Markov chains, the token generation process of self-attention is non-Markovian as it depends on the whole past trajectory. Thus, our work initiates the statistical and consistency study of learning self-attention process by highlighting its unique nature and challenges.

\section{Discussion}

In this work, we have studied theoretical properties of the self-attention layer by formally linking its dynamics to (context-conditioned) Markov chains. Through this connection, we identify when a ground-truth self-attention layer is learnable by observing its output tokens. We develop consistency and finite sample learning guarantees for multiple prompts as well as for single trajectory learning, which reveal novel insights into the self-attention mechanism (such as prompt coverage conditions and distribution collapse in the single trajectory). 



An important future direction is relaxing Assumption \ref{assume iden} to more general and realistic conditions. This assumption implicitly necessitates that $d \geq K$ and also samples the next token from the tokens within the input sequence. An initial way to relax this assumption is to assume that $\Cb \Eb^\top$ is equal to a column stochastic matrix. More broadly, instead of linear classifier $\Cb$, it is possible to incorporate a Multi-Layer Perceptron (MLP) into the model. However, these relaxations may cause the loss of convexity but provide a deeper understanding of the self-attention layer. The stochastic matrix assumption provides flexibility to have a next token that is not inside the input prompt. In this model, we think that the equivalency between the \PCMC model and the attention can be constructed and consistency can be obtained with minor modifications in the definition of the co-occurrence graphs. However, the sample complexity guarantees should be established in a different way as the convexity does not hold for general stochastic matrices. In addition to relaxing the assumption on $\Cb$, it is also interesting to study the case of $d \ll K$, which is related to the low-rank adaption (LoRA) of the attention matrix. We believe that the correspondence of this attention matrix is a compressed version of the base Markov chain.

Other possible future directions are (1) studying the multi-layer attention models and their connection to hierarchical Markov models and representation learning, (2) characterizing the consistency and finite sample learnability of self-attention from the single trajectory, and (3) analysis of the impact of the End-Of-Sequence (EOS) token, which is utilized to terminate the generation of outputs in modern language models.
 


\section*{Acknowledgement} 
This work was supported in part by the NSF grants CCF-2046816 and CCF-2212426, UMich's MIDAS PODS program, a Google Research Scholar award, and an Adobe Data Science Research award. The authors thank Vijay Subramanian for helpful discussion.

\newpage
\bibliography{main}
\bibliographystyle{icml2024}

\onecolumn
\appendix
\section{Proof of Theorems in Section \ref{sect setup}}\label{app sect setup}
In this section, we share the proofs of Lemmas \ref{lemma identity}, \ref{lemma Stoken}, and Theorem \ref{theorem equivalency ps}. 
\subsection{Proof of Lemma \ref{lemma identity}}\label{sec proof identity}
\begin{lemma}[Restated Lemma \ref{lemma identity}]\label{lemma app identity}
    Recall the probability vector $\bpi^X\in\R^K$ from Definition \ref{def pcmc}. We have that
\[ 
f_{\W}(X)=\X^\top\s_X=\Eb^\top \bpi^X.
\]
\end{lemma}
\begin{proof}
Suppose $\X = \M\Eb$ where $\M \in \R^{T \times K}$ is a universal mapping matrix, which specifies the token index for each entry. Specifically, $M_{jk} = 
                \begin{cases}
                    1, & \x_j = \eb_k \\ 
                    0, & \x_j \neq \eb_k
                \end{cases}$. Note that $\M^{\top}\mathbf{1}_{L} = \m(X) = \m$. Then, we have:
\begin{equation}
\begin{split}
    \X^{\top}\s_X = \Eb^{\top}\M^{\top}\s_X
\end{split}
\end{equation}
Then, it is equivalent to prove $(\M^{\top}\s_X)_k = \bpi^X_k$ for any $k \in [K]$. Let $s_0 = \sum_{j \in [K]} \exp(\eb_j^{\top}\W\xl)$. To proceed, using the definition of $\s_X$, we get:
\begin{equation}
\begin{split}
    (\M^{\top}\s_X)_k &= \frac{m_k \cdot \exp (\eb_k^{\top}\W\xl)}{\sum_{j \in [K]} m_j \cdot \exp(\eb_j^{\top}\W\xl)} \\ 
    &= \frac{m_k \cdot \exp (\eb_k^{\top}\W\xl) / s_0}{\sum_{j \in [K]} m_j \cdot \exp(\eb_j^{\top}\W\xl) / s_0} \\ 
    &= \frac{m_k \cdot \pi_{\xxl, k}}{\sum_{j \in [K]} m_j \cdot \pi_{\xxl, j}} \\ 
    &= \bpi^X_k
\end{split}
\end{equation}
which completes the proof. 
\end{proof}
\subsection{Proof of Lemma \ref{lemma Stoken}}
\begin{lemma}[Restated Lemma \ref{lemma Stoken}]\label{lemma app Stoken}
    For all $\W \in \R^{d \times d}$ and $X$: $f_{\W}(X) = f_{\bPi_{\Sall}(\W)}(X)$.
\end{lemma}
\begin{proof}
    Let $\Sall^{\perp}$ be the orthogonal complement of the subspace $\Sall$ in $\R^{d \times d}$. Then, for any $\W \in \R^{d \times d}$, we have $\W = \Pi_{\Sall}(\W) + \Pi_{\Sall^{\perp}}(\W)$. 
    
    By definition of $\Sall$, there exists $c \in \R$ such that $\eb_i^\top \Pi_{\Sall^{\perp}}(\W) \eb_j = c$ for every $i,j \in [K]$. 
    
    As a result, using the definition of $f_{\W}(X)$ in \eqref{eq:sattn-def}, we obtain that
    \begin{align*}
        f_{\W}(X) = \X^T \sft{\X \W \xl} &= \X^T \sft{\X (\Pi_{\Sall}(\W) + \Pi_{\Sall^{\perp}}(\W)) \xl} \\
        &= \X^T \sft{\X \Pi_{\Sall}(\W) \xl + c\boldsymbol{1}_L } \\
        &= \X^T \sft{\X \Pi_{\Sall}(\W) \xl} \\
        &= f_{\Pi_{\Sall}(\W)}(X)
    \end{align*}
    which completes the proof.
\end{proof}
\subsection{Proof of Theorem \ref{theorem equivalency ps}}
\begin{theorem}[Restated Theorem \ref{theorem equivalency ps}]\label{theorem app equivalency ps}
    Suppose Assumption \ref{assume iden} holds. Let $\Pc$ be the set of transition matrices with non-zero entries. For each $\Pb\in \Pc$, there is a unique $\W\in \Sall$ with $\Pbw=\Pb$. Thus, for any prompt $X\in[K]^L$ and next token $y=x_{L+1}\in[K]$
    \begin{align*}
        \P_{\Pb}(y | X ) = \cb_{y}^\top \X^\top \sft{\X \W \xl}
    \end{align*}
    where $\cb_y$ is the $y^{\text{th}}$ row of the linear prediction head $\Cb$.
\end{theorem}
\begin{proof}
1) We prove that  there exists $\W \in \Sall$ satisfying the lemma. Using Lemma \ref{lemma identity}, for any $\W \in \Sall$, let $\bpi^{\W}_i = \sft{\Eb\W\eb_i}$, we have:
\begin{equation}
    \cb_{y}^\top \X^\top \sft{\X \W \xl} = \cb_{y}^\top \Eb^{\top}\bpi^X = \bpi^X_y = \frac{m_y \cdot \pi_{\xxl, y}^{\W}}{\m^{\top}\bpi_{\xxl}^{\W}}
\end{equation}
Comparing the equation above with
\[\P_{\Pb}(y | X ) = \frac{m_y \cdot \pi_{\xxl, y}}{\m^{\top}\bpi_{\xxl}},
\]
it is sufficient to prove that for any given $\bpi$, there exists a solution $\W$ for the following problem:
    \begin{equation}
         \bpi =  \exp(\Eb\W\xb)
    \end{equation}
    It is equivalent to solving the following linear system:
    \begin{equation}
        \Eb\w = \dot \bpi
    \end{equation}
    where $\w = \W \xb, \dot \bpi = \log \bpi$.
    Since the rows of $\Eb$ are linearly independent from Assumption~\ref{assume iden}, $\Eb$ is right invertible, implying that there exists at least one solution to the problem above. 

    2) We prove the uniqueness as follows: Let's assume the inverse. There exists $\W_1 , \W_2 \in \Sall$ such that $\W_1 \neq \W_2$ and we have the following for any $(X,y)$:
    \begin{align}\label{uniqueness}
         \cb_{y}^\top \X^\top \sft{\X \W_1 \xl} = \cb_{y}^\top \X^\top \sft{\X \W_2 \xl}
    \end{align}
    As $\W_1, \W_2 \in \Sall$ and $\W_1 \neq \W_2$, there exists $i,j,k \in [K]$ such that $\Pi_{(\eb_i - \eb_j) \eb_k^\top}(\W_1 - \W_2) > 0$. 
    Then, let's consider $X = [i, j]$ and $y = k$ and $\xl = k$. (If the attention model is self-attention, then we can include $k$ into $X$ as well). Let $\s_{X, \W_1} = \sft{\X \W_1 \xl}$ and $\s_{X, \W_2} = \sft{\X \W_2 \xl}$. Let $u_{mn} = \eb_m^\top \W_n \eb_k$ for $n \in \{1,2\} m \in \{i,j\}$. As we have $\Pi_{(\eb_i - \eb_j) \eb_k^\top}(\W_1 - \W_2) > 0$, then we obtain that $u_{j1} - u_{i1} \neq u_{j2} - u_{i2}$, which implies that $\s_{X, \W_1} = \s_{X, \W_2}$. As $\Cb \Eb^\top = \boldsymbol{I}$ by Assumption \ref{assume iden}, \eqref{uniqueness} cannot hold, which is a contradiction. This completes the proof. 
\end{proof}
\section{Proof of Theorems in Section \ref{sect consistency}}\label{app sect consistency}
\begin{figure*}[tb] 
    \centering    \includegraphics[width=0.9\textwidth]{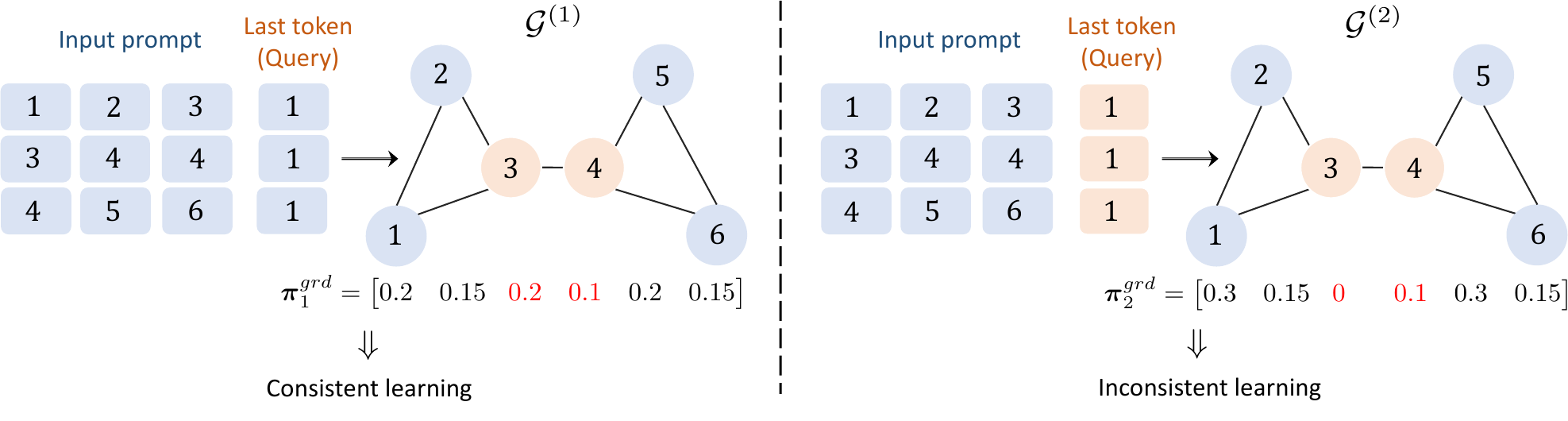}
    \vspace{-10pt}
    \caption{\small{{Illustration of the Co-occurrence Graphs for the Cross-attention Models with Possible Zero Transition Probability in $\Pgrd$. In this figure, we draw the co-occurrence graphs of the input prompt distribution whose support consists of three elements under the cross-attention model. In the left figure, the first column of the transition matrix $\Pgrd$ does not include any zero transition whereas, in the right figure, the first column of the transition matrix $\Pgrd$ includes zero transition. For the left figure, the co-occurrence graph $\Gc^{(1)}$ is connected with respect to $\Pgrd$. However, in the right figure, the co-occurrence graph $\Gc^{(1)}$ is not connected with respect to $\Pgrd$ as there is no path between $2^{\text{nd}}$ vertex and the $5^{\text{th}}$ vertex using non-zero vertex. }
    }}
    \label{fig:consistency_app}
\end{figure*}
In this section, we will analyze the case where the ground-truth transition matrix $\Pgrd$ has zero transitions. Note that the transition matrix that has zero probability transitions is not important for the attention models as the equivalency between the attention models and the \PCMC model is constructed for the non-zero transitions. We provide a detailed analysis for the sake of the Markov chain community. First, we share a supplementary lemma for this section: 
\begin{lemma}\label{lemma supporting observability}
    Define the function $f : \R_{>0}^{n} \xrightarrow[]{} \R$ for fixed $(y_i)_{i=1}^n$ positive variables and define $\x^*$ as follows:
    \begin{align}
        \x^* = \arg &\min_{\x \in \R_{>0}^n} f(\x) := \arg \min_{\x \in \R_{>0}^n} -\sum_{i=1}^n y_i \log(x_i) \\
        &\text{subject to} \quad \sum_{i=1}^n x_i = 1, \qquad 
    \end{align}
    Then, we have $c \in \R$ such that $c = \frac{x^*_i}{y_i}$ for all $i \in [n]$.
\end{lemma}
\begin{proof}
    Let $\mathcal{L}(\x, \lambda)$ be the Lagrangian function:
    \begin{align*}
        \mathcal{L}(\x, \lambda) = -\sum_{i=1}^n y_i \log(x_i) + \lambda\left( 1- \sum_{i=1}^n x_i\right) 
    \end{align*}
    As a result of KKT condition, we have
    \begin{align}
        -\frac{y_i}{x^*_i} - \lambda = 0 \quad \forall i \in [n] 
    \end{align}
    which completes the proof.
\end{proof}
\begin{definition}\label{defn app connected}
     \noindent $(i) $ Let $\Omega_k$ be the set of input prompts that are inside the support of $\Dcx$ and whose queries are the $k^{\text{th}}$ token. Then, we define an undirected co-occurrence graph $(\Gc^{(k)})_{k=1}^K$ with $K$ vertices such that the vertices $i, j \in [K]$ are connected in $\Gc^{(k)}$ if there exists an input prompt in $\Omega_k$ that includes both the $i^{\text{th}}$ and $j^{\text{th}}$ tokens.

     \noindent $(ii) $ Let $\Pgrd$ be the ground-truth transition matrix. For an arbitrary query $k \in [K]$, the co-occurrence graph $\Gc^{(k)}$ is said to be `connected with respect to $\Pgrd$' if it satisfies the following: For every pair of vertices $(i,j) \in [K] \times [K], \Pgrdnb_{ki} \neq 0, \Pgrdnb_{kj} \neq 0$, there exists a path of $(v_m)_{m=1}^M$ such that $v_1 = i$, $v_M = j$, and $\Pgrdnb_{kv_m} \neq 0$ for every $m \in [M]$.
\end{definition}

To explain the notion of the `connected graph with respect to $\Pgrd$', we demonstrate a dataset for different $\Pgrd$ in Figure \ref{fig:consistency_app}. Note that the following theorem reduces to Theorem \ref{theorem consistency} if $\Pgrd$ has non-zero transition probabilities. 
\begin{theorem}[Stronger version of Theorem \ref{theorem consistency}] \label{theorem consistency app}
    Let $\Pgrd$ be the transition matrix of a Markov chain that determines the next token under the \PCMC model. Let $(\Gc^{(k)})_{k=1}^K$ be the co-occurrence graphs based on the input prompt distribution $\Dcx$. Then, the estimation of $\Pgrd$ in \eqref{defn Pstr} with the prompt distribution $\Dcx$ is consistent if and only if $\Gc^{(k)}$ is connected with respect to $\Pgrd$ (see Definition \ref{defn app connected} (ii)) for every $k \in [K]$. 
\end{theorem}

\begin{proof}
     For an arbitrary $k \in [K]$, we are going to prove that the $k^{\text{th}}$ column of $\Pgrd$ and $\Pst$ are equivalent if and only if $\Gc^{(k)}$ is connected with respect to $\Pgrd$. The proof of this statement is sufficient to prove Theorem \ref{theorem consistency app}.

      Let $\bpist$ and $\bpigrd$ be the $k^{\text{th}}$ column of $\Pst$ and $\Pgrd$, respectively. Let $\Omega_k$ be the set of input prompts such that the query is the $k^{\text{th}}$ token. Let $\Sp$ be the set of token IDs $i$ such that $\Pgrdnb_{ki} \neq 0$. Let $X$ be an arbitrary token inside the set $\Omega_k$. Let $[X]$ represent the set of token IDs inside the input prompt $X$. We change the notation of $\P_{\Pb}(y | X)$ to $\P_{\bpi}(y|X)$ as we will deal with the input prompts whose queries are the same. We first minimize the population risk for this specific input sequence $X$. 
    \begin{align}
        \bpist^X &= \arg\min_{\bpi} \E_y \left[- \log \P_{\bpi}(y|X)\right] \nonumber \\
        &=\arg\min_{\bpi} - \sum_{i \in \Sp} \P_{\bpigrd}(y = i | X) \log \left( \P_{\bpi}(y = i| X)\right) \label{lemma apply}
    \end{align}
    Applying Lemma \ref{lemma supporting observability} on \eqref{lemma apply} where $K= n$, $y_i = \P_{\bpigrd}(y = i | X)$ and $x_i = \P_{\bpi}(y=i| X)$, we know that there exists $c_1 \in \R$ such that
    \begin{align}
        \frac{\P_{\bpigrd }(y = i | X)}{\P_{\bpist^X}(y=i| X)} = c \qquad \forall i \in \Sp
    \end{align}
    From \eqref{eqn masked markov chain}, we know that the output probabilities of the \PCMC model are a linear transformation of the transition probabilities based on the occurrences of tokens inside an input prompt. As a result, the linear transformation is one-to-one for the tokens IDs $i$ such that $i \in \Sp \cap [X]$. Then, there exists $c_2 > 0$ such that 
    \begin{align}\label{Socc equal}
        \frac{(\pist^\X)_i }{(\pigrd)_i } = c_2 \qquad \forall i \in \Sp \cap [X]
    \end{align}

    Note that $\bpist = \pigrd$ satisfies \eqref{Socc equal} for all $i \in [K]$. This means that $\bpigrd$ is a solution to the population risk minimization problem in \eqref{defn Pstr}. What is remaining is that there is no other $\bpi \neq \bpigrd$ such that $\bpi$ is a solution to the population risk minimization problem. As $\bpi = \bpigrd$ satisfies \eqref{Socc equal} for all $i \in \Sp \cap [X]$, if there exists any other $\bpi$ that minimizes \eqref{defn Pstr}, then this $\bpi$ should satisfy \eqref{Socc equal} for all possible $\Omega_k$. 
    
    \textbf{Proof of $\Leftarrow$}: Now, we know that $\Gc^{(k)}$ is connected with respect to $\Pgrd$ and we want to prove the consistency of estimation. Let's assume the inverse: There exists a probability vector $\bpi \in \R^K$ such that for an arbitrary index $\bar{i} \in [K]$ we have $\pi_{\bar{i}} > \pigrd_{\bar{i}}$ and $\pi$ minimizes \eqref{defn Pstr}. Then, there must exist $\bar{j} \in [K]$ such that $\pi_{\bar{j}} < \pigrd_{\bar{j}}$ as both $\bpi$ and $\bpigrd$ are probability vectors. This implies that 
    \begin{align}\label{eqn app consistency conradiction}
        \frac{\pi_{\bar{i}}}{\pigrd_{\bar{i}} } > 1 > \frac{\pi_{\bar{j}}}{\pigrd_{\bar{j}} }
    \end{align}
    Since $\Gc^{(k)}$ is connected with respect to $\Pgrd$, there exists a path of $(v_m)_{m=1}^M  \in \Gc^{(k)}$ such that $v_1 = \bar{i}$ and $v_M = \bar{j}$ and $v_m \in \Sp$ for every $m \in [M]$. Since the path $(v_m)_{m=1}^M$ is inside the co-occurrence graph $\Gc^{(k)}$, there exists an input prompt sequence $(X_m)_{m=1}^{M-1}$ such that $X_m \in \Omega_k$ and $X_m$ includes the tokens with ID $v_{m}$ and $v_{m+1}$. Using \eqref{Socc equal} for the input prompt sequence $(X_m)_{m=1}^{M-1}$, we obtain the following using the fact that these input sequences include the $k^{\text{th}}$ token:
    \begin{align}\label{eqn app consistency apply all input prompts}
        \frac{\pi_{v_m} }{\pigrd_{v_m} } = \frac{\pi_{v_{m+1}} }{\pigrd_{v_{m+1}} }\qquad \forall m \in [M-1]
    \end{align}
    Combining \eqref{eqn app consistency apply all input prompts} and the fact that $v_1 = \bar{i}$ and $v_{M-1} = \bar{j}$, we obtain
    \begin{align*}
        \frac{\pi_{\bar{i}}}{\pigrd_{\bar{i}} } = \frac{\pi_{\bar{j}}}{\pigrd_{\bar{j}} }
    \end{align*}
    which contradicts with \eqref{eqn app consistency conradiction}. This means that if $\Gc^{(k)}$ is connected with respect to $\Pgrd$, then the only solution that minimizes \eqref{defn Pstr} is $\bpigrd$, which is equivalent to the consistency of estimation in Definition \ref{defn consistency of estimation}.

    \textbf{Proof of $\Rightarrow$}:  Now, we know that the estimation of $\bpigrd$ in \eqref{defn Pstr} is consistent and we want to prove that $\Gc^{(k)}$ is connected with respect to $\Pgrd$. Let's assume the inverse: There exists $i, j \in \Sp$ such that there is no path between them in $\Gc^{(k)}$ using the vertices in $\Sp$. Then, there exists a partition $\Sc_1$ and $\Sc_2$ such that $\Sc_1 \cap \Sc_2 = \emptyset$, $\Sc_1 \cup \Sc_2 = \Sp$, $i \in \Sc_1$, $j \in \Sc_2$, and there is no path from any element of $\Sc_1$ to any element $\Sc_2$ in $\Gc^{(k)}$ using the vertices in $\Sp$. We construct a probability vector $\bpi \in \R^{K}$ such that $\bpi$ minimizes \eqref{defn Pstr} and $\bpi \neq \bpigrd$:
    \begin{align*}
        & \bar{\bpi} \in \R^K \\
        &\bar{\pi}_k = 2 \pigrd_k \qquad \forall k \in \Sc_1 \\
        &\bar{\pi}_k = \pigrd_k \qquad \forall k \in \Sc_2 \\
        &\pi_k = \Bigg\{
    \begin{array}{lr}
        \frac{\bar{\pi}_k}{\sum_{k'=1}^K \bar{\pi}_{k'}}, & \text{if } k \in \Sp \\
        0, & \text{if } k \not\in \Sp
    \end{array}
    \end{align*}
    Note that this constructed $\bpi$ satisfies \eqref{Socc equal} for all $X \in \Omega_k$, which implies that $\bpi$ is a minimizer of $\eqref{defn Pstr}$ and $\bpi \neq \bpigrd$. This is a contradiction to the consistency of estimation, which completes the proof.
    \end{proof}

    \subsection{Strict Convexity and Smoothness of Loss Function}\label{sect app strict convexity}
    In this subsection, we discuss the convexity, strict convexity, and smoothness of the loss functions. First, we analyze the empirical loss function $\Lch_n(\W)$, then we connect our analysis with the population loss function $\Lc(\W)$. Throughout the section, we omit the subscript of $n$ in $\Lch_n(\W)$ as we are analyzing the loss function for any arbitrary $n$. First, we define the following subspace:
    \begin{definition}\label{def svm + cyc subspace}
    Define the subspace $\Sprm$ as the span of all matrices $(\eb_i-\eb_j)\eb_k^\top$ for all $i, j, k \in [K]$ such that there exists an input prompt in the dataset that includes the $i^{\text{th}}$ token, $j^{\text{th}}$ token is the next token, and $k^{\text{th}}$ token is the query or the last token. 
\end{definition}
The following Lemma proves the strict convexity inside a subspace, which is a slightly generalized version of Lemma 9 in \cite{anonymous}. Even though the proofs are almost the same as \cite{anonymous}, we restate the proofs for the sake of completeness and notational coherence.
\begin{lemma}[Stronger version of Lemma \ref{lemma strict convexity consistency}, \cite{anonymous}]\label{lemma cvx app}
    Suppose Assumptions \ref{assume iden} and \ref{assume cooccurrence connected} hold. Then $\Lc(\W)$ and $\Lch(\W)$ is convex on $\R^{d \times d}$. Furthermore, $\Lc(\W)$ and $\Lch(\W)$ are strictly convex on $\Sall$ and $\Sprm$, respectively.
\end{lemma}
\begin{proof}
First, we are going to prove the convexity and the strict convexity for $\Lch(\W)$. Then, we apply our findings to $\Lc(\W)$. Recall from Definition \ref{defn Stoken} that $\Sall$ is the span of all matrices $(\eb_i - \eb_j) \eb_k^\top$ for $i,j,k \in [K]$. 

\noindent$\bullet$ \textbf{First Case: $\W \in \Sall$.} Let $g: \Sall \xrightarrow[]{} \R^{K \times K}$ such that $g(\W) = \Eb \W \Eb^{\top}$. By definition, this function is linear. In addition to that, this function $g$ is invertible on $g(\Sall)$ by Assumption \ref{assume iden} and the domain of the function is $\Sall$. Note that Assumption \ref{assume iden} ensures $\text{rank}(\Eb) = K$.

Let  $\Eb' = \Cb' = \Iden_k$, $(\X'_i, y_i')_{i=1}^n$ be a dataset constructed from $(\X_i, y_i)_{i=1}^n$ such that $y_i' = y_i$ and $\X_i' = \X_i \Eb^{\dagger}$. Then, for any $\W' \in \R^{K \times K}$, we have the following:
\begin{align*}
    \Lch \circ g^{-1}(\W') = \frac{1}{n}\sum_{i=1}^n -\log\left((\cb'_{y_i})^\top (\X_i')^\top \sft{\X_i' \W' \xli'}\right)
\end{align*}

Using Lemma \ref{lemma linear map strict convexity} and \ref{lemma vect strict convexity}, we know that $\Lch \circ g^{-1}(\W')$ is convex on $\R^{K \times K}$ and strictly convex on $g(\Sprm)$. Using these two facts and Lemma \ref{lemma linear map strict convexity}, we have $\Lch(\W)$ is convex on $\Sall $ and strictly convex on $\Sprm \cap \Sall = \Sprm$. 

\noindent$\bullet$ \textbf{Second Case: $\W \not\in \Sall$.} Using Lemma \ref{lemma Stoken}, we have the following for any $0 \leq \lambda \leq 1$:
\begin{align}\label{linearity of projection}
    \Lch(\lambda \W_1 + (1-\lambda)\W_2 ) = \Lch(\lambda \Pi_{\Sc_K}(\W_1) + \lambda \Pi_{\Sc_K}(\W_2))  
\end{align}
Then, using \eqref{linearity of projection}, we have the following:
\begin{align*}
    \lambda \Lch(\W_1) + (1-\lambda) \Lch(\W_2) &= \lambda \Lch(\Pi_{\Sall} (\W_1)) + (1-\lambda) \Lch(\Pi_{\Sall}(\W_2)) \\
    & \stackrel{(a)}{\geq} \Lch(\lambda \Pi_{\Sall}(\W_1) + \lambda \Pi_{\Sall}(\W_2)) = \Lch(\lambda \W_1 + (1-\lambda)\W_2 )
\end{align*}
where (a) follows from the convexity of $\Lch(\W)$ inside $\Sall$. This implies that $\Lch(\W)$ is convex when $\W \not\in \Sall$. Note that $\Sprm \subset \Sall$, therefore we do not look at the strict convexity in this case. 

For the loss function $\Lc(\W)$, the same procedure can be applied. By Assumption \ref{assume cooccurrence connected}, the subspace of $\Sprm$ for the population dataset becomes $\Sall$ as $\Gc^{(k)}$ are connected for every $k \in [K]$. 
\end{proof}

\begin{lemma}[Lemma 10, \cite{anonymous}]\label{lemma linear map strict convexity}
    Let $T : \Xc \xrightarrow[]{} \Yc$ be an invertible linear map. If a function $f : \Yc \xrightarrow[]{} \R$ is convex/strictly convex on $\Yc$, then $f \circ T (x)$ is a convex/strictly convex function on $\Xc$. 
\end{lemma}
\begin{proof}
    Let $x_1 \neq x_2 \in \Xc$ be arbitrary variables. Let $y_1 = T(x_1)$ and $y_2 = T(x_2)$. Since $T$ is an invertible map, $y_1 \neq y_2$. Since $T$ is a linear map, $ T(\lambda x_1 + (1-\lambda)x_2) = \lambda y_1 + (1-\lambda) y_2$ for $0 < \lambda < 1$. Then, we obtain the following
    \begin{align*}
        \lambda (f \circ T(x_1)) + (1-\lambda) (f \circ T(x_2)) &= \lambda f (y_1) + (1-\lambda) f(y_2) \\
        & \stackrel{(a)}{>} f(\lambda y_1 + (1-\lambda) y_2) \\ 
        &= f \circ T(\lambda x_1 + (1-\lambda) x_2)
    \end{align*}
    where (a) follows from the strict convexity of the function $f$. This implies that $f \circ T (x)$ is a strictly convex function on $\Xc$. Note that if $y_1 = y_2$, then we cannot achieve (a). Additionally, if $f$ is convex instead of strictly convex, then $>$ in (a) is changed to $\geq$, and $f \circ T(x)$ is convex.
\end{proof}

\begin{lemma}[Lemma 11, \cite{anonymous}]\label{lemma vect strict convexity}
    Let $\Eb = \boldsymbol{I}_d$. Let $f : \R^{d \times d} \xrightarrow[]{} \R^{d^2}$ be a linear transformation defined as $f(\W) = \vb$ where $v_{i\times d + j} = {\eb_i^T \W \eb_j}$. Then, $\Lch \circ f^{-1}(\vb)$ is convex. Furthermore, $\Lch \circ f^{-1}(\vb)$ is strictly convex on $f(\Sprm)$, where $\Sprm$ is defined in Definition~\ref{def svm + cyc subspace}.
\end{lemma}
\begin{proof}
    \noindent$\bullet$ \textbf{We first prove that  $\Lc \circ f^{-1}(\vb)$  is convex. }
Let $\ell: \R^{d^2} \times \R^{T \times d} \times \R \xrightarrow[]{} \R$ be defined as follows:
    \begin{align*}
        \ell(\vb, \X, y) := -\log\left(\cb_{y}^\top \X^\top \sft{\X \big(f^{-1}(\vb)\big) \xl}\right).  
    \end{align*}
    Then, we have the following:
    \begin{align}\label{defn ellbar strict convexity}
        \Lch \circ f^{-1}(\vb) = \frac{1}{n}\sum_{i=1}^n -\log\left(\cb_{y_i}^\top \X_i^\top \sft{\X_i \big(f^{-1}(\vb)\big) \xli}\right) = \frac{1}{n}\sum_{i=1}^n \ell(\vb, \X_i, y_i).  
    \end{align}
    Note that the summation of convex functions is convex. Therefore, it is sufficient to prove the convexity of $\Lc \circ f^{-1}(\vb)$ by proving the convexity of $\bar{\ell}(\vb, \X, y)$ for an arbitrary pair of input sequence and label $(\X, y)$. For the simplicity of notation, we use $\ell(\vb)$ instead of $\ell(\vb, \X, y)$. Let $k$ be the last token of $\X$. 
    By Assumption \ref{assume iden} and log-loss, we know that
    \begin{align*}
       \ell(\vb) := \ell(\vb, \X, y) &= -\log\left(\frac{ \Fc(\X)_y \cdot e^{v_{y\times d + k}}}{\sum_{j\in[K]} \Fc(\X)_j \cdot e^{v_{j\times d + k}}} \right) \\
       &= \log\left(\sum_{j \in [K]} \Fc(\X)_y \cdot e^{v_{j\times d + k}}\right) -\log (\Fc(\X)_y \cdot e^{v_{y\times d + k}}).
   \end{align*}   
   Let $\z \in \R^{d^2}$ be a vector such that the $(j\times d + k)^{\text{th}}$ element of $\z$ is $z_{j \times d +k}= \Fc(\X)_j \cdot e^{v_{j\times d +k}}$ for $k \in [K]$, otherwise $z_i = 0$. Then, the Hessian matrix of $\bar{\ell}(\vb)$ is 
\begin{align*}
    \nabla^2 \ell(\vb) = \frac{1}{(\boldsymbol{1}^\top \z)^2}\left((\boldsymbol{1}^\top \z) \text{diag}(\z) - \z \z^\top \right)
\end{align*}
For any $\ub \in \R^{d^2}$, we obtain that 
\begin{align}\label{CSI}
    \ub^\top \nabla^2 \ell(\vb) \ub = \frac{1}{(\boldsymbol{1}^\top \z)^2} \left( \left(\sum_{j=1}^{d^2} z_j \right) \left(\sum_{j=1}^{d^2} u_j^2 z_j\right)  - \left(\sum_{j=1}^{d^2} u_j z_j \right)^2 \right) \geq 0  .
\end{align}
Since $\z_i\geq0$, $i\in[d^2]$, \eqref{CSI} follows from the Cauchy-Schwarz inequality $(\bal^\top \bal)(\bt ^\top \bt) \geq (\bal^\top \bt)^2$ applied to the vectors with $\alpha_i = u_i \sqrt{z_i}$ and $\beta_i = \sqrt{z_i}$. The equality condition holds $k \bal = \bt$ for $k \neq 0$. This means that $\ell(\vb)$ is convex.

\noindent$\bullet$ \textbf{Next, we will show that  $\Lch \circ f^{-1}(\vb)$ is strictly convex on $f(\Sprm)$.} Assume that $\Lch \circ f^{-1}(\vb)$ is not strictly convex on $f(\Sprm)$. Using the convexity of $\Lch \circ f^{-1}(\vb)$, this implies that there exist $\ub, \vb \in f(\Sprm)$, $\|{\ub}\|_2 > 0$ such that 
\begin{align*}
    \ub^\top \left( \nabla^2 \Lch \circ f^{-1}(\vb) \right) \ub = 0
\end{align*}
Combining this with the convexity of $\ell(\vb)$ and \eqref{defn ellbar strict convexity}, we have the following:
\begin{align}\label{strict convexity necessary condition}
    \ub^{\top} \left( \nabla^2 \ell(\vb, \X_i, \y_i) \right) \ub = 0 \qquad \forall i \in [n]
\end{align}

Now, we are going to prove that $\|\ub\|_2 = 0$ if \eqref{strict convexity necessary condition} holds. As $\ub \in f(\Sprm)$, there exists $\W \in \Sprm$ such that $f(\W) = \ub$. As the function $f$ preserves the norm, $\tf{\W} > 0$. By definition of $\Sprm$, there exist $\bar{i},\bar{j}, \bar{k} \in [K]$ and $(\X_{\bar{n}}, y_{\bar{n}}) \in \data$ such that $\<(\eb_{\bar{i}} - \eb_{\bar{j}}) \eb_{\bar{k}}^T, \W \> > 0$, $\X_{\bar{n}}$ includes the $\bar{j}^{\text{th}}$ token, the last token of $\X_{\bar{n}}$ is the $\bar{k}^{\text{th}}$ token, and $y_{\bar{n}} = \bar{i}$. On the other hand, by the generation of the dataset $\data$, the next token should be inside the input prompt. Then, $z_{\bar{i}\times d + \bar{k}}$ and $z_{\bar{j} \times d + \bar{k}}$ in \eqref{CSI} are non-zero for this input sequence $\X_{\bar{n}}$. Using the equality condition of Cauchy-Schwartz Inequality in \eqref{CSI}, we obtain that $u_{\bar{i} \times d + \bar{k}} - u_{\bar{j} \times d + \bar{k}} = 0$. This implies that 
\begin{align*}
    0 &= u_{\bar{i} \times d + \bar{k}} - u_{\bar{j} \times d + \bar{k}} \\
    &= \eb_{\bar{i}}^T \W \eb_{\bar{k}} - \eb_{\bar{j}}^T \W \eb_{\bar{k}} \\
    &= (\eb_{\bar{i}} - \eb_{\bar{j}})^T \W \eb_{\bar{k}} = \<(\eb_{\bar{i}} - \eb_{\bar{j}}) \eb_{\bar{k}}^T, \W \>
\end{align*}
which contradicts with the fact that $\| \ub \|_2 > 0$. This completes the proof.
\end{proof}

\section{Proof Theorems in Section \ref{sect finite sample}}\label{app sect sample complexity}
In this section, we first share supplementary lemmas for the proof of Theorem \ref{theorem sample complexity}.
\subsection{Supplementary Lemmas for Proof of Theorem \ref{theorem sample complexity}}

\begin{definition}
    Let $\B(\Wst, r) \subset \R^{d \times d}$ be a ball centered at a point $\Wst$ with radius $r$ defined as follows:
    \begin{align*}
        \B(\Wst, r) = \{ \W \in \Sall \quad |  \quad \tf{\W - \Wst} \leq r \}
    \end{align*}
\end{definition}
\begin{lemma}\label{lemma Lipschitz Sample Complexity}
    Suppose that Assumption \ref{assume iden} holds. Then, for any $(X, y)$ where the token ID $y$ exists in $X$, and  $\W \in \B(\Wst, r)$ the absolute loss difference satisfies the following:
    \begin{align*}
        |\ell(\cb_{y}^\top\X^\top \sft{\X\W\xl}) - \ell(\cb_{y}^\top\X^\top \sft{\X\Wst\xl})| \leq 2r \max_{i \in [K]} \tnt{\eb_i}^2
    \end{align*}
    This implies that the loss function is $2 \max_{i \in [K]} \tnt{\eb_i}^2-$Lipschitz.
\end{lemma}
\begin{proof}
    Let $\xo$ be the token occurrence vector, i.e., $\xo_i$ is the number of occurrences of the $i^{\text{th}}$ token inside the input sequence $\X$. Let $x_L$ be the last token ID of $\X$. Let $\ub \in \R^K$ be the vector such that $u_k = \eb_k^T \W \eb_{x_L}$ for $k \in [K]$. Then, the loss function will be the following:
    \begin{align*}
        \ell(\cb_{y}^\top\X^\top \sft{\X\W\xl}) = -\log \left( \frac{\xo_y \e^{u_y}}{\sum_{i=i}^K \xo_i \e^{u_i}} \right) = - \log ( \xo_y ) - u_y + \log \left( {\sum_{i=i}^K \xo_i \e^{u_i}}\right)
    \end{align*}
    Let $\ub^*$ be the vector such that $\ub^*_k = \eb_k^T \Wst \eb_{x_L}$ for $k \in [K]$. Then, the loss difference will be the following:
    \begin{align*}
        \ell(\cb_{y}^\top\X^\top \sft{\X\W\xl}) - \ell(\cb_{y}^\top\X^\top \sft{\X\Wst\xl}) &= \log \left( {\sum_{i=i}^K \xo_i \e^{u_i}}\right) - u_y  - \log \left( {\sum_{i=i}^K \xo_i \e^{u^*_i}}\right) + u_y^* \\ 
        &\leq \left| \left( {\sum_{i=i}^K \xo_i \e^{u_i}}\right) - \log \left( {\sum_{i=i}^K \xo_i \e^{u^*_i}}\right) \right| + |u_y - u_y^*|
    \end{align*}
    Let $f:\R^{K} \xrightarrow[]{} \R$ be defined as $f(\ub) = \log \left( {\sum_{i=i}^K \xo_i \e^{u_i}}\right)$. Then, the derivative of the function $f$ is the following:
    \begin{align*}
        \nabla_j f(\ub) = \frac{\xo_j \e^{u_j}}{\sum_{i=1}^K \xo_i \e^{u_i}} \qquad \forall j \in [K]
    \end{align*}
    Note that $f$ is a continuous function. Using the Intermediate Value Theorem, for any $\ub, \ub^*$, there exists $\vb \in \R^K$ such that $\vb \in [\ub, \ub^*]$ and it satisfies the following:
    \begin{align*}
        |f(\ub) - f(\ub^*)| = \nabla f(\vb)^T (\ub - \ub^*) \leq \sum_{j=1}^K \nabla_j f(\vb) \| \ub - \ub^* \|_{\infty} = \| \ub - \ub^* \|_{\infty}
    \end{align*}
    As a result, we obtain that
    \begin{align*}
        | \ell(\cb_{y}^\top\X^\top \sft{\X\W\xl}) - \ell(\cb_{y}^\top\X^\top \sft{\X\Wst\xl}) | &\leq \left| \left( {\sum_{i=i}^K \xo_i \e^{u_i}}\right) - \log \left( {\sum_{i=i}^K \xo_i \e^{u^*_i}}\right) \right| + |u_y - u_y^*| \\        
        &\leq 2 \ti{\ub - \ub^*} \\
        &= 2\max_{i \in [K]} \eb_i^T (\W - \Wst) \eb_{x_L} \\
        & \leq 2r \max_i \tnt{\eb_i}^2 
    \end{align*}
\end{proof}
\begin{lemma}\label{lemma fixed W concentration}
    Let $\W \in \B(\Wst, r)$ be an arbitrary attention matrix. Then, for any $\Wst \in \R^{d \times d}$ and $\delta > 0$, we have the following with probability at least $1- 2 \delta$
    \begin{align*}
         |\Lc(\W) - \Lc(\Wst) - \Lch(\W) + \Lch(\Wst) | <  r \max_{i \in [K]} \tnt{\eb_i}^2 \sqrt{\frac{8 \log(1/\delta)}{n}}   
    \end{align*}
\end{lemma}
\begin{proof}
    We are going to utilize McDiarmid's Inequality. First, we check whether the assumption of McDiarmid's Inequality is satisfied.  Let $\data = ((\X_i, \xli, y_i))_{i=1}^n$ be the dataset. Let $\Xc_i$ be the sample space of $(\X_i, \xli, y_i)$. Let the function $f : \Xc_1 \times \Xc_2 \times \dots \Xc_n \xrightarrow[]{} \R$ be defined as follows:
    \begin{align*}
        f(\data) := \Lch(\W) - \Lch(\Wst) = \frac{1}{n}\sum_{i=1}^n \ell(\cb_{y_i}^\top\X_i^\top \sft{\X_i\W\xli}) - \frac{1}{n}\sum_{i=1}^n \ell(\cb_{y_i}^\top\X_i^\top \sft{\X_i\Wst\xli})
    \end{align*}
    Let $\data' = ((\X_i', \xli', y_i'))_{i=1}^n$ be the dataset such that the samples are different for only the $j^{\text{th}}$ sample from $\data$. In other words, if $j \neq i$, then $(\X_i', \xli', y_i') = (\X_i, \xli, y_i)$. Then, we are going to look at the following difference:
    \begin{align*}
        |f(\data) - f(\data')| &= 
        \frac{1}{n}\Bigg| \ell(\cb_{y_j}^\top\X_j^\top \sft{\X_j\W\xlj{j}}) -  \ell(\cb_{y_j}^\top\X_j^\top \sft{\X_j\Wst\xlj{j}}) \\
        &\hspace{0.5cm} - \Big( \ell(\cb_{y_j'}^\top\X_j'^\top \sft{\X_j'\W\xlj{j'}}) -  \ell(\cb_{y_j'}^\top\X_j'^\top \sft{\X_j'\Wst\xlj{j}'})\Big)  \Bigg| \\
        &\leq \frac{1}{n} \Bigg| \ell(\cb_{y_j}^\top\X_j^\top \sft{\X_j\W\xlj{j}}) -  \ell(\cb_{y_j}^\top\X_j^\top \sft{\X_j\Wst\xlj{j}}) \Bigg| \\
        & \hspace{0.5cm}+ \frac{1}{n} \Bigg| \ell(\cb_{y_j'}^\top\X_j'^\top \sft{\X_j'\W\xlj{j}'}) -  \ell(\cb_{y_j'}^\top\X_j'^\top \sft{\X_j'\Wst\xlj{j}'}) \Bigg| \\
        &\stackrel{(a)}{\leq} \frac{4 r \max_{i \in [K]} \tnt{\eb_i}^2}{n}
    \end{align*}
    where (a) follows from Lemma \ref{lemma Lipschitz Sample Complexity}. Now, using McDiarmid's Inequality, we obtain the following for any $\eps > 0$:
    \begin{align*}
        \P \left( |\E[\Lch(\W)] - \E[\Lch(\Wst)] - \Lch(\W) + \Lch(\Wst)| > \eps \right) &\leq 2 \exp \left(\frac{2\eps^2}{\sum_{j=1}^n \frac{16 r^2 \max_{i \in [K]} \tnt{\eb_j}^4}{n^2}}\right) \\
        &= 2\exp \left(\frac{ n \eps^2 }{8 r^2 \max_{i \in [K]} \tnt{\eb_i}^4} \right)
    \end{align*}
    This completes the proof.
\end{proof}
\begin{lemma}\label{lemma app covering}
    For any $\Wst \in \R^{d \times d}$, and $\delta > 0$, we have the following with probability at least $1 - 2\delta$
    \begin{align*}
        \sup_{\W \in \B(\Wst, r)} |\Lc(\W) - \Lc(\Wst) - &\Lch(\W) + \Lch(\Wst)| \\
        &\leq \inf_{r > \eps > 0} \left\{ 2\eps \max_{i \in [K]} \tnt{\eb_i}^2 + r \max_{i \in [K]} \tnt{\eb_i}^2  \sqrt{\frac{8 K^2\log((r + 2/\eps)/\delta)}{n}} \right\}
    \end{align*}
\end{lemma}
\begin{proof}
    For any $\eps > 0$, let $\Ace : \Nc(\Sall, \eps)$ be the minimal $\eps$ cover of $\Sall \cap \B(\Wst, r)$ in terms of the Frobenius norm. Let the function $g : \Sall \times \Sall \xrightarrow[]{} \R$ be defined as $g(\W, \W') = \Lc(\W) - \Lc(\W') - (\Lch(\W) - \Lch(\W'))$. Then, we have the following:
    \begin{align*}
        \sup_{\W \in \B(\Wst, r)} | g(\W, \Wst) | \leq \sup_{\W \in \B(\Wst, r)} \min_{\W' \in \Ace} |g(\W, \W')| + \max_{\W \in \Ace}|g(\W, \Wst)| 
    \end{align*}
    By definition, there exists $\W' \in \Ace$ such that $\W' \in \B(\W, \eps)$. Then, using Lemma \ref{lemma Lipschitz Sample Complexity}, we have the following:
    \begin{align*}
        \sup_{\W \in \B(\Wst, r)} \min_{\W' \in \Ace} |g(\W, \W')| \leq 2 \eps \max_{i \in [K]} \tnt{\eb_i}^2
    \end{align*}
    On the other hand, note that the cardinality of $\Ace$ is finite and it is upper bounded by $(r + 2/\eps)^{K^2}$ from Corollary 4.2.13 \cite{vershynin2018high}. Then, we apply union bound to Lemma \ref{lemma fixed W concentration} and obtain the following with probability at least $1-2\delta$:
    \begin{align*}
        \max_{\W \in \Ace}|g(\W, \Wst)| &\leq r \max_{i \in [K]} \tnt{\eb_i}^2 \sqrt{\frac{8 \log(|\Ace|/\delta)}{n}} \\
        &\leq r \max_{i \in [K]} \tnt{\eb_i}^2  \sqrt{\frac{8 K^2\log((r + 2/\eps)/\delta)}{n}}
    \end{align*}
    Combining all of the results, for any $\eps > 0$, we derive the following with probability at least $1-2\delta$
    \begin{align*}
        \sup_{\W \in \B(\Wst, r)} |g(\W, \Wst)| \leq  2\eps \max_{i \in [K]} \tnt{\eb_i}^2 + r \max_{i \in [K]} \tnt{\eb_i}^2  \sqrt{\frac{8 K^2\log((r + 2/\eps)/\delta)}{n}} 
    \end{align*}
    which completes the proof.
\end{proof}
\subsection{Proof of Theorem \ref{theorem sample complexity}}\label{sect app sample complexity}
\begin{theorem}[Restated Theorem \ref{theorem sample complexity}]\label{theorem app sample complexity}
    Suppose Assumptions \ref{assume iden} and \ref{assume cooccurrence connected} hold. Let $R_0>0$ be a finite constant based on the structure of $\Wgrd$ and $\Dcx$. Then, if $n \geq R_0 K^2$, with probability at least $1-2\delta$
    \begin{align*}
        \Lc(\Wh_n) - \Lc(\Wst) \lesssim  \frac{K^2\log\frac{n}{K\delta}}{n}.
    \end{align*}
\end{theorem}
\begin{proof}
    Let $r_0 \in \R$ and consider the ball $\B(\Wst, r_0)$. From Lemma \ref{lemma cvx app}, we know that $\Lc(\W)$ is strictly convex on $\Sall$. Additionally, note that the function $\Lc(\W)$ is differentiable twice and its second derivative is continuous. As the $\B(\Wst, r_0)$ is a compact set, there exists a positive constant $\alpha > 0$ such that the eigenvalues of the Hessian of $\Lc(\W)$ are lower bound by $\alpha$. It means that the loss function $\Lc(\W)$ is $\alpha-$strongly convex in the ball $\B(\Wst, r_0)$. 

    Now, let $\emax := \max_{i \in [K]} \tnt{\eb_i}$ and let's set the following values:
    \begin{align}
        &\eps = \frac{2K\emax^2}{\alpha n} \label{eps set}\\
        &D_{\delta} = \sqrt{8 K^2 \log(3/(\eps \delta))} \label{ddelta set} \\
        &r = \frac{2^5 D_{\delta} \emax^2}{\alpha \sqrt{n}} \label{r set} \\
        &n \geq \frac{2^{11} D_{\delta}^2 \emax^4 }{\alpha^2 (\min(r_0, 1))^2} \label{n set}
    \end{align}
    From \eqref{eps set} and \eqref{r set}, we have $\eps < r/\sqrt{n}$, which shows that we can utilize this $\eps$ in Lemma \ref{lemma app covering}. From \eqref{r set} and \eqref{n set}, we have $r \leq \min(r_0, 1)$, which implies that the loss function $\Lc(\W)$ is $\alpha-$strongly convex inside the ball $\B(\Wst, r)$. We are going to show that $\Wh_n \in \B(\Wst, r)$ if $n$ satisfies \eqref{n set}.  Let $\Wout$ be an arbitrary point on the boundary of the ball $\B(\Wst, r)$. Let $\Winn$ be the boundary point of the ball $\B(\W, r/2)$ such that it is on the line segment between $\Wst$ and $\Wout$. By the strong convexity of $\Lc(\W)$ on $\B(\Wst, r)$, we have the following:
    \begin{align}\label{app sample complexity strong convexity}
        \Lc(\Wout) - \Lc(\Winn) \geq \frac{\alpha r^2}{8}
    \end{align}
    We apply Lemma \ref{lemma app covering} to both $\Winn$ and $\Wout$. As $r \leq \min(r_0, 1)$, we have $\log((r+2/\eps)/\delta)< \log(3/(\eps \delta))$. Then, we obtain the following for any $\eps > 0$ with probability at least $1-2\delta$
    \begin{align}
        |\Lc(\Winn) - \Lc(\Wst) - \Lch(\Winn) + \Lch(\Wst)| &\leq 2\eps \emax^2 + r \emax^2  \sqrt{\frac{8 K^2\log((r + 2/\eps)/\delta)}{n}} \nonumber \\
         & \stackrel{(a)}{\leq} \frac{4 D_{\delta} r \emax^2}{\sqrt{n}} \label{Winn applied lemma}
         \end{align}
         where (a) follows from the fact that $\eps \leq D_{\delta} r / \sqrt{n}$. Similarly, we have the following with probability at least $1- 2\delta$
         \begin{align}
        |\Lc(\Wout) - \Lc(\Wst) - \Lch(\Wout) + \Lch(\Wst)| &\leq 2\eps \emax^2 + r \emax^2  \sqrt{\frac{8 K^2\log((r + 2/\eps)/\delta)}{n}} \\
        &\leq \frac{4 D_{\delta} r \emax^2}{\sqrt{n}} \label{Wout applied lemma}
    \end{align}
    Combining \eqref{app sample complexity strong convexity} and \eqref{Wout applied lemma}, we obtain the following for any $\eps > 0$
    \begin{align}\label{Wout Winn geq }
        \Lch(\Wout) \geq \Lc(\Winn) -\Lc(\Wst) + \Lch(\Wst) + \frac{\alpha r^2}{8} - \left( \frac{4 D_{\delta} r \emax^2}{\sqrt{n}} \right)
    \end{align}
    From \eqref{Winn applied lemma}, we obtain the following:
    \begin{align}\label{Wout Winn leqq}
        \Lch(\Winn) \leq \Lc(\Winn) -\Lc(\Wst) + \Lch(\Wst) + \left( \frac{4 D_{\delta} r \emax^2}{\sqrt{n}} \right)
    \end{align}
    From \eqref{ddelta set}, \eqref{r set}, and \eqref{n set}, we have that
    \begin{align}\label{Wout Winn alpha}
        \frac{\alpha r^2}{4} > \frac{8 D_{\delta} r \emax^2}{\sqrt{n}}
    \end{align}
    Combining \eqref{Wout Winn geq }, \eqref{Wout Winn leqq}, and \eqref{Wout Winn alpha}, we obtain the following 
    \begin{align}\label{Wout Winn final}
        \Lch(\Wout) > \Lch(\Winn)
    \end{align}
    Note that \eqref{Wout Winn final} is valid for any boundary point of $\B(\Wst, r)$. Due to the convexity of $\Lch(\W)$ from Lemma \ref{lemma cvx app}, $\Wh_n$ should be inside the ball $\B(\Wst, r)$ with probability at least $1-2\delta$ as a result of \eqref{Wout Winn final}. Now, let's use the smoothness of the loss function $\Lc(\W)$. From \cite{anonymous}, we know that $\Lc(\W)$ is $2 \emax^2 \sqrt{L_{max}}-$smooth. As $\Wh_n$ is inside the ball $\B(\Wst, r)$,  we have the following with at least probability $1-2\delta$:
    \begin{align}
        \Lc(\Wh_n) - \Lc(\W) \leq \frac{\emax \sqrt{L_{max}} r^2}{2} \stackrel{(a)}{\lesssim}  \frac{K^2\log\frac{n}{K\delta}}{n}
    \end{align} 
    where (a) follows from \eqref{r set}. This completes the proof. 
\end{proof}
\subsection{Proof of Corollary \ref{corollary sample complexity}}\label{corollary sample complexity app}
\begin{corollary}[Restated Corollary \ref{corollary sample complexity}]
        Consider the setting in Theorem \ref{theorem sample complexity} and suppose Assumptions \ref{assume iden}, \ref{assume stoken consistency}, and \ref{assume cooccurrence connected} hold. Then, if $n \geq R_0 K^2$, with probability at least $1-2\delta$
    \begin{align}
        \tf{\Wh_n - \Wgrd}^2 \lesssim \frac{K^2\log\frac{n}{K\delta}}{n}.
    \end{align}
\end{corollary}
\begin{proof}
    We show in the proof of Theorem \ref{theorem app sample complexity} that with probability at least $1-2\delta$ that $\Wh_n \in \B(\Wst, r)$.

    With the same argument in the proof of Theorem \ref{theorem app sample complexity}, we have the strong convexity inside $\B(\Ws, r_0)$: From Lemma \ref{lemma cvx app}, we know that $\Lc(\W)$ is strictly convex on $\Sall$. Additionally, note that the function $\Lc(\W)$ is differentiable twice and its second derivative is continuous. As the $\B(\Wst, r_0)$ is a compact set, there exists a positive constant $\alpha > 0$ such that the eigenvalues of the Hessian of $\Lc(\W)$ are lower bound by $\alpha$. It means that the loss function $\Lc(\W)$ is $\alpha-$strongly convex in the ball $\B(\Wst, r_0)$. Recall that $r \leq \min(r_0, 1)$ from \eqref{eps set}, \eqref{ddelta set}, \eqref{r set}, and \eqref{n set}. Therefore, the loss function $\Lc(\W)$ is $\alpha-$strongly convex in the $\B(\Wst, r)$ as well. 

    By $\alpha-$strong convexity of $\Lc(\W)$ on $\B(\Wst, r)$, we obtain that
    \begin{align*}
        \Lc(\Wst) - \Lc(\Wh_n) \geq \frac{\alpha \tf{\Wst - \Wh_n}}{2} 
    \end{align*}
    In addition to that, the estimation of $\Wgrd$ is consistent by Theorem \ref{theorem consistency} using Assumption \ref{assume cooccurrence connected}. This implies that $\Wgrd = \Wst$. This completes the proof. 
\end{proof}
\section{Proof of Theorems in Section \ref{sect single traj}}\label{app sect single traj}
\subsection{Proof of lemma \ref{lem inf token visit}}
\begin{definition} Given sample space $\Omega = \{0, 1\}^{\infty}$ and let $A_{k, t} = \mathbf{1}(y_t = k)$. 
Let $\mathcal{F}_{k,t}, t \geq 0$ be a filtration with $\mathcal{F}_{k,0} = \{\emptyset, \Omega\}$ and $\mathcal{F}_{k,t} \coloneqq \sigma(A_{k,j} | j \in [t])$, where $\sigma$ refers to the $\sigma$-algebra.
\end{definition}
\begin{lemma}[Restated Lemma \ref{lem inf token visit}]\label{lem app inf token visit} Let $\Pgrd$ be a transition matrix with non-zero entries. For all $k\in[K]$, $\Pro(\lim_{n\rightarrow \infty}S_{k,n}=\infty)=1$.
\end{lemma}

\begin{proof}
Let $\bar \pi_{k} = \min_{i \in [K]} \pi_{ik}$. Consider any token $k \in [K]$, 
Then we have: 
\begin{equation}
\begin{split}
    \text{Pr}(A_{k, t}|\mathcal{F}_{k, t-1}) \stackrel{(a)}{=} \P(y_t = k) = \sum_{i = 1}^{K} \P(y_t = k | \xlp_t = i) p(\xlp_t = i)
\end{split}
\end{equation}
To proceed, consider the minimum of  $\P(y_t = k | \xlp_t = i)$: 
\begin{equation}
\begin{split}
     \min_{i \in [K]} \P(y_t = k | \xlp_t = i)  & =\min_{i \in [K]}\frac{\pi_{ik} S_{k, t}}{S_{k, t} (2\pi_{ik}-1) + (t + 2)(1 - \pi_{ik})} \\ 
        & = \min_{i \in [K]} \frac{\pi_{ik} }{ (2\pi_{ik}-1) + \frac{(t + 2)(1 - \pi_{ik})}{S_{k, t}}} \\ 
        &\geq \min_{i \in [K]} \frac{\pi_{ik} }{ (2\pi_{ik}-1) + {(t + 2)(1 - \pi_{ik})}} \\ 
        &= \min_{i \in [K]} \frac{\pi_{ik} }{t + 1 - t \pi_{ik}} \\ 
        & \stackrel{(a)}\geq \frac{\bar \pi_k}{t}
\end{split}
\end{equation}
where (a) holds for sufficiently large $t$ where $t \geq 1 / \bar \pi_k$. Then we have:
\begin{equation}
\begin{split}
    \text{Pr}(A_{k, t}|\mathcal{F}_{k, t-1}) &\geq \frac{\bar \pi_{k}}{t} \sum_{i = 1}^{K}  p(\xlp_t = i) = \frac{\bar \pi_{k}}{t}  \\           
\end{split}
\end{equation}

Hence, the sum of the probabilities over infinite iterations diverges to infinity, i.e., $\sum_{t=1}^{\infty} \text{Pr}(A_{k, t}|\mathcal{F}_{k, t-1}) = \infty$. Using the second Borel-Cantelli lemma, we have:
\begin{align}
    \text{Pr} \Big( \limsup_{t \to \infty} A_{k, t} \Big) = 1
\end{align}
which reveals that token $k$ can be observed infinitely many times in the trajectory. This implies $\Pro(\lim_{n\rightarrow \infty}S_{k, n}=\infty)=1 \text{ where }S_{k, n} = \sum_{t = 1}^n A_{k, t}$.
\end{proof}
\subsection{Proof of lemma \ref{lem dist col}}
\begin{lemma}[Restated Distribution collapse Lemma \ref{lem dist col}]\label{lem app dist col}
Consider the \PCMC model with $K = 2$ defined in Section \ref{sec:emp-single}. Suppose that $\X_1$ includes all vocabulary at least once. Recall that $\Fc(X_t)$ denotes the empirical frequency of individual states where $X_t$ is the state trajectory at time $t$. {For any $t > t_0$ with a sufficiently large $t_0$}, we have:
    \begin{align*}
        \E[\m(X_t)_2] < t^{-q}
    \end{align*}
    where $q = 1 - p / (1-p)$. Furthermore, when $p < 1/2$, 
    \begin{align*}
        \lim_{t \to \infty} \E\left[\frac{\Fc(X_t)_2}{\Fc(X_t)_1}\right]=0.
    \end{align*}
\end{lemma}
\begin{proof}

Let $S_t$ and $L_t$ be the number of token 2 in the input prompt and the length of the input prompt at iteration $t$, respectively. Note that $S_1 \geq 1$. Then, at iteration $t$, we have:
\begin{equation}
\begin{split}
\mathbb{E}[S_{t}] &= S_{t-1} + \frac{S_{t-1} p }{S_{t-1} p + (L_t - S_{t-1})(1 - p)} \\ 
&= S_{t-1}(1 + \frac{p }{S_{t-1} (2p-1) + L_t(1 - p)})\\
\end{split}
\end{equation}
Since the probability of selecting a specific token is weighted by the number of occurrences, when $p < 1/2$, the model tends to sample token $1$ over $2$. Moreover, due to the positive reinforcement nature, selecting token $1$ as the label will further increase the probability of selecting token $1$ in the next round. Thus, for any $p < 1/2$, there exists a sufficiently large $t_0$ such that when $t > t_0, L_t \gg S_{t-1}$. Hence: 
\begin{equation}
\begin{split}
\mathbb{E}[S_{t}] &= S_{t-1}(1 + \frac{p }{S_{t-1} (2p-1) + L_t(1 - p)})\\
&\approx S_{t-1}(1 + \frac{p }{L_t(1 - p)}) \\ 
&= S_{t-1}(1 + \bar p /L_t)
\end{split}
\end{equation}
where $\bar p = p / ({1-p})$. 
Recall that $\Fc(X_t)_2 = \frac{S_t}{L_t}$. For $t > t_0$ where $t_0$ is sufficiently large, we have:
\begin{equation}
\begin{split}
      \mathbb{E}[\Fc(X_t)_2]  &= \E \left[ \frac{S_t}{L_t} \right] \\
        &= \E \left[ \frac{S_{t-1}(1 + \bar p/L_t)}{L_{t-1}(1 + 1/L_{t-1})} \right]  \\ 
        &= \mathbb{E}[\Fc(X_{t-1})_2] \cdot \frac{1 + \bar p/L_t}{1 + 1/L_{t-1}} \\
        &\approx \mathbb{E}[\Fc(X_{t-1})_2] (1 - q/L_t) 
\end{split}
\end{equation}
where $q \coloneqq 1 - \bar p = (1-2p)/{(1-p)}$. Applying this equation recursively, we get:
\begin{equation}
     \mathbb{E}\left[ \frac{\Fc(X_t)_2}{\Fc(X_1)_2} \right] = \prod_{\tau=1}^{t}(1 - q/\tau) \stackrel{(a)} \leq  \prod_{\tau=1}^{t} \exp({-q/\tau}) = \exp({-q\sum_{\tau=1}^t 1/\tau}) \stackrel{(b)}<\exp({-q \ln t}) = t^{-q}
\end{equation}
where (a) comes from $1 + x \leq e^{x}$ for any $x$ and (b) comes from the fact that $\sum_{\tau=1}^t 1/\tau > \ln t$ when $t$ is large. By Euler-Maclaurin formula, $H_t = \sum_{\tau=1}^t 1/\tau = \ln t + \gamma + {1}/{(2t)} - \varepsilon_t  \approx \ln t + \gamma \geq \ln t$ where $\gamma \approx 0.5772$. As a result, for a sufficiently big $t$, $\E[\Fc(X_t)_2] < t^{-q}$ as $\Fc(X_1)_2$ is a finite number. Moving forward, since $\E[\Fc(X_t)_2] + \E[\Fc(X_t)_1] = 1$, we have:
\begin{equation}
    \E[\frac{\Fc(X_t)_2}{\Fc(X_t)_1}] = \E[\frac{\Fc(X_t)_2}{1 - \Fc(X_t)_2}] = \frac{\E[\Fc(X_t)_2]}{1 - \E[\Fc(X_t)_2]} < \frac{t^{-q}}{2 - t^{-q}} = \frac{1}{2t^{q} - 1}
\end{equation}
As $p < 1/2$, $q \in (0, 1)$, when $t \to \infty$, the ratio $\lim_{t \to \infty} \frac{1}{2t^q - 1}$ goes to zero. Combining the fact that $\frac{\Fc(X_t)_2}{\Fc(X_t)_1} \geq 0$, it implies:
\begin{equation}
    \lim_{t \to \infty}\E[\frac{\Fc(X_t)_2}{\Fc(X_t)_1}]  = 0
\end{equation}
\end{proof}
\section{Proof of theoretical results in Section \ref{sect position embedding}}\label{app sect position}

\subsection{Proof of lemma \ref{lemma positional embedding}}
\begin{lemma}[Restated Lemma \ref{lemma positional embedding}]\label{lemma app positional embedding}
    Suppose Assumptions \ref{assume iden} and \ref{assume positional embedding C} hold. Then, for any $\W \in \R^{d \times d}$, there exists the transition matrix $\Pb^{\W}$ such that for any $(X, y)$ we have the following:
    \begin{align*}
        \P_{\Pb, \Ub}( y | X) = \cb_{y}^\top \X^\top \sft{\X \W \xl}
    \end{align*}
\end{lemma}
\begin{proof}
Consider any $\W \in \R^{d \times d} $ and an arbitrary input $(X,y)$ with positional embedding $\Ub$. Let $\X = \M\Eb + \Ub$, where $\M$ is the same mapping matrix defined in Section \ref{sec proof identity}. Define $        \ab \coloneqq \exp(\Ub\W\ub_{L}), 
        \bb \coloneqq \exp(\Eb\W\ub_{L})  ,
        \Vb \coloneqq \exp(\Ub\W\Eb^{\top}) $ we have:
\begin{equation}\label{eq:pos-attn}
\begin{split}
\cb_{y}^\top \X^\top \sft{\X \W \xl} &= \cb_{y}^{\top}(\Eb^{\top}\M^{\top} + \Ub^{\top})\sft{\X\W\xl} \\ 
&\stackrel{(a)}= \cb_{y}^{\top}\Eb^{\top}\M^{\top}\sft{\X\W\xl} \\ 
&\stackrel{(b)}= \frac{\sum_{i = 1}^L \exp (\x_y^{\top}\W\xl) \cdot \mathbf{1}(x_i = y)}{\sum_{k = 1}^K \sum_{i = 1}^L \exp (\x_k^{\top}\W\xl) \cdot \mathbf{1}(x_i = k)} \\ 
&=\frac{ \sum_{i = 1}^L \exp \Big((\eb_y + \ub_y)^{\top}\W(\eb_{\xxl} + \ub_{\xxl})\Big) \cdot \mathbf{1}(x_i = y)}{\sum_{k = 1}^K \sum_{i = 1}^L \exp \Big((\eb_{k} + \ub_{k})^{\top}\W(\eb_{\xxl} + \ub_{\xxl})\Big) \cdot \mathbf{1}(x_i = k)} \\ 
&=\frac{b_y \cdot \exp(\eb_y^{\top}\W\eb_{\xxl}) \sum_{i = 1}^L a_i \cdot V_{i, \xxl}\cdot \mathbf{1}(x_i = y)}{\sum_{k=1}^K b_k \cdot \exp(\eb_k^{\top}\W\eb_{\xxl}) \sum_{i = 1}^L a_i \cdot V_{i, \xxl}  \cdot \mathbf{1}(x_i = k)} \\ 
\end{split}
\end{equation}
where (a) comes from the assumption that $\Cb\Ub^{\top} = \mathbf{0}$, (b) comes from the assumption that $\Cb\Eb^{\top} = \Ib_K$ and the definition of $\M$. Comparing Equation~\eqref{eq:pos-attn} with Equation~\eqref{pos PCMC transition}, we can set $\pi_k = \exp(\eb_k^{\top}\W\eb_{\xxl}) / \sum_{j \in [K]}\exp(\eb_j^{\top}\W\eb_{\xxl})$ to obtain $\Pb^{\W}$.
\end{proof}

\end{document}